\documentclass[nohyperref]{article}

\usepackage{microtype}
\usepackage{graphicx}
\usepackage{subfigure}
\usepackage{booktabs} 
\usepackage{caption}
\usepackage{graphicx}
\usepackage{url}            

\usepackage{hyperref}

\usepackage[accepted]{icml2022}

\usepackage{amsmath}
\usepackage{amssymb}
\usepackage{mathtools}
\usepackage{amsthm}
 \usepackage[suppress]{color-edits}
 \addauthor{sw}{blue}
 \addauthor{kh}{red}
 \addauthor{ls}{DarkGreen}
 \addauthor{dn}{magenta}
 \addauthor{hh}{cyan}

\usepackage[capitalize,noabbrev]{cleveref}

\theoremstyle{plain}
\newtheorem{theorem}{Theorem}[section]
\newtheorem{proposition}[theorem]{Proposition}
\newtheorem{lemma}[theorem]{Lemma}
\newtheorem{corollary}[theorem]{Corollary}
\theoremstyle{definition}
\newtheorem{definition}[theorem]{Definition}

\theoremstyle{remark}

\newtheorem{example}[theorem]{Example}
\newtheorem{fact}[theorem]{Fact}

\newcommand{\E}{\mathop{\mathbb{E}}}

\newcommand{\vx}{\mathbf{x}}

\newcommand{\va}{\mathbf{a}}
\newcommand{\vu}{\mathbf{u}}

\newcommand{\vtheta}{\boldsymbol{\theta}}
\newcommand{\vb}{\mathbf{b}}

\def\epsilon{\varepsilon}

\newcommand{\xhdr}[1]{\vspace{-1mm} \noindent{\bf #1}}

\usepackage[textsize=tiny]{todonotes}

\icmltitlerunning{Strategic Instrumental Variable Regression}

\usepackage{enumitem}

\begin{document}

\twocolumn[
\icmltitle{Strategic Instrumental Variable Regression:\\ Recovering Causal Relationships From Strategic Responses}



\icmlsetsymbol{equal}{*}

\begin{icmlauthorlist}
\icmlauthor{Keegan Harris}{cmu}
\icmlauthor{Daniel Ngo}{equal,umn}
\icmlauthor{Logan Stapleton}{equal,umn}
\icmlauthor{Hoda Heidari}{cmu}
\icmlauthor{Zhiwei Steven Wu}{cmu}
\end{icmlauthorlist}

\icmlaffiliation{cmu}{School of Computer Science, Carnegie Mellon University, Pittsburgh, USA}
\icmlaffiliation{umn}{Computer Science Department, University of Minnesota, Minneapolis, USA}

\icmlcorrespondingauthor{Keegan Harris}{keeganh@cs.cmu.edu}

\icmlkeywords{Machine Learning, ICML}

\vskip 0.3in
]



\printAffiliationsAndNotice{\icmlEqualContribution} 

\begin{abstract}
In settings where Machine Learning (ML) algorithms automate or inform consequential decisions about people, individual decision subjects are often incentivized to strategically modify their observable attributes to receive more favorable predictions. As a result, the distribution the assessment rule is trained on may differ from the one it operates on in deployment.
 While such distribution shifts, in general, can hinder accurate predictions, our work identifies a unique opportunity associated with shifts due to strategic responses: We show that we can use strategic responses effectively to recover \emph{causal} relationships between the observable features and outcomes we wish to predict, even under the presence of unobserved confounding variables. 
 Specifically, our work establishes a novel connection between strategic responses to ML models and instrumental variable (IV) regression by observing that the sequence of deployed models can be viewed as an \emph{instrument} that affects agents' observable features but does not \emph{directly} influence their outcomes.
 We show that our causal recovery method can be utilized to improve decision-making across several important criteria: individual fairness, agent outcomes, and predictive risk. In particular, we show that if decision subjects differ in their ability to modify non-causal attributes, any decision rule
deviating from the causal coefficients can lead to
(potentially unbounded) individual-level unfairness.
\end{abstract}

\section{Introduction}\label{sec:intro}

Machine learning (ML) predictions increasingly inform high-stakes decisions for people in areas such as college admissions~\citep{pangburn2019,Somvichian2021}, credit scoring~\citep{whitecase,rice2013discriminatory}, employment~\citep{sanchez2020does}, and beyond. One of the major criticisms against the use of ML in socially consequential domains is the failure of these technologies to identify \emph{causal} relationships among relevant attributes and the outcome of interest~\citep{kusner2017counterfactual}. The single-minded focus of ML on predictive accuracy has given rise to brittle predictive models that learn to rely on spurious correlations---and at times, harmful stereotypes---to achieve seemingly accurate predictions on held-out test data~\citep{sweeney2013discrimination,kusner2020long}. The resulting models frequently underperform in deployment, and their predictions can negatively impact decision subjects. 
As an example of the long-term negative consequences of ML-based decision-making systems, they often prompt individuals to modify their observable attributes \emph{strategically} to receive more favorable predictions---and subsequently, decisions~\citep{hardt2016strategic}. These strategic responses are among the primary causes of distribution shifts (and subsequently, the unsatisfactory performance) of ML in high-stakes decision-making domains. Moreover, recent work has established the potential of these tools to amplify existing social disparities by \emph{incentivizing different effort investments} across distinct groups of subjects~\citep{liu2020disparate,heidari2019long,mouzannar2019fair}.

The above challenges have led to renewed calls on the ML community to strengthen their understanding of the connections between ML and causality~\citep{pearl2019seven,scholkopf2019causality}. Knowledge of causal relationships among predictive attributes and outcomes of interest promotes several desirable aims: First, ML practitioners can use this knowledge to debug their models and ensure robustness even if the underlying population shifts over time.
Second, policymakers can utilize the causal understanding of a domain in their policy choices and examine a decision-making system's compliance with policy goals and societal values (e.g., they can audit the system for unfairness against particular populations~\citep{loftus2018causal}).   
Finally, predictions rooted in causal associations block undesirable pathways of gaming and manipulation and, instead, encourage decision subjects to make meaningful interventions that improve their actual outcomes (as opposed to their assessments alone).

Our work responds to the above calls by offering a new approach to recover causal relationships between observable features and the outcome of interest in the presence of strategic responses---without substantially hampering predictive accuracy. We consider settings where a decision-maker deploys a sequence of models to predict the outcome for a sequence of strategic decision subjects. Often in high-stakes decision-making settings such as the ones mentioned earlier, there are unobserved confounding variables that influence subjects' attributes and outcomes simultaneously. Our key observation is that we can correct for the effect of such confounders by viewing the sequence of \textbf{assessment rules as valid \emph{instruments}} which affect subjects' observable features but do not \emph{directly} influence their outcomes. Our main contribution is a general framework that recovers the causal relationships between observed attributes and the outcome of interest by treating assessment rules as instruments.

\subsection{Our Setting}

Next, we describe our theoretical setup in further detail, then proceed to an overview of our findings. For concreteness, we utilize a stylized university admissions scenario as our running example for the remainder of this section. However, the reader should note that our model is applicable to other real-world applications in which confounders taint the causal interpretation of predictive models. For example, in credit lending, lack of access to affordable credit affects not only the applicant's debt, but also their likelihood of default~\citep{collard2005affordable}.
In university admissions (which will be our running example), research has shown that the socioeconomic background of a student can impact both their SAT scores and success in college~\citep{sackett2009socioeconomic}. 

With the running example in mind, 
consider a stylized setting in which a university decides whether to admit or reject applicants on a rolling basis\footnote{See \citet{psbwebsite} for a list of such universities in the United States.} based (in part) on how well they are predicted to perform if admitted to the university (See \Cref{fig:intro-dags}). We model such interactions as a game between a \emph{principal} (here, the university) and a population of \emph{agents} (here, university applicants) who arrive sequentially over $T$ rounds, indexed by $t=1,2, \cdots, T$. 
In each round $t$, the principal deploys an {assessment rule $\vtheta_t \in \mathbb{R}^m$, which is used to assign agent $t$ a predicted outcome $\widehat{y}_t \in \mathbb{R}$}. In our running example, $\widehat{y}$ could correspond to the applicant's predicted college GPA if admitted. The predicted outcome is calculated based on certain observable/measured attributes of the agent, denoted by $\vx_t \in \mathbb{R}^m$. For example, in the case of a university applicant, these attributes may include the applicants' standardized test scores, high school math GPA, science GPA, humanities GPA, and their extracurricular activities. For simplicity, we assume all assessment rules are \emph{linear}, that is, $\hat{y}_t = \vx_t^\top \vtheta_t + \hat{o}_t$ for all $t$. (Where $\hat{o}_t$ is the current estimate of the expected offset term $\E[o_t]$.) This linear setup corresponds to an instance of the \emph{partially linear regression model} (originally due to \citet{robinson1988root}), a commonly studied setting in both the causal inference and strategic learning literature (e.g., \citet{shavit2020strategic, kleinberg2020classifiers, BechavodLWZ21}).

\xhdr{Measured vs. latent variables.}
We assume that the agent best-responds to the assessment rule $\vtheta_t$ by strategically modifying their observable attributes $\mathbf{x}_t$ to receive a more favorable predicted outcome. Often agents cannot modify the value of their measured attributes (e.g., SAT score) directly, but only through investing effort in certain activities that are difficult to measure. 
 For example, a student might take standardized test preparation courses to improve their SAT scores, or they may spend time studying the respective subjects to improve their math and humanities GPA. 

\xhdr{Latent variables: effort investments.} We formalize the above hidden investments with a vector $\mathbf{a}_t \in \mathbb{R}^d$, capturing the unobservable efforts agent $t$ invests in $d$ \emph{activities} in response to the assessment rule $\vtheta_t$. We assume there exists a linear mapping $\mathcal{E}_t$ which translates efforts to changes in observable attributes for agent $t$. The $(k,j)$-th entry of this effort conversion matrix defines the change in the $k$-th observable attribute of agent $t$, $\vx_{t}$, for one unit increase in the $j$th coordinate of their effort vector $\mathbf{a}_t$.

\begin{figure*}
    \centering
    \includegraphics[width=\textwidth]{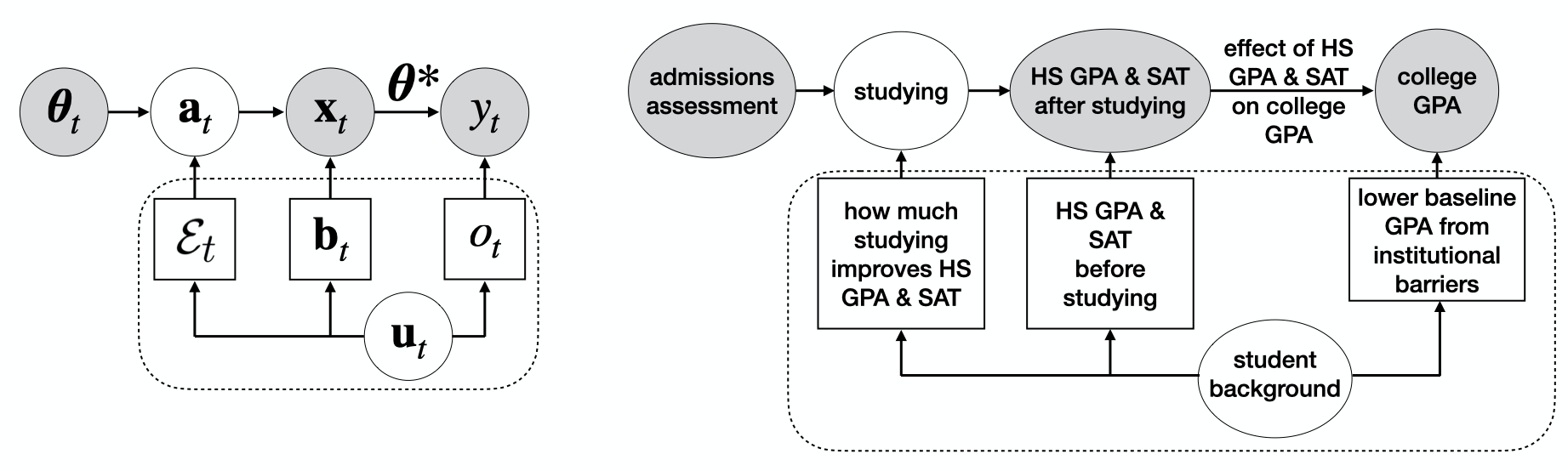}
    \caption{Graphical model for our setting (left) along with the way it corresponds to the admissions running example (right). Grey nodes are observed, white unobserved. Observable features $\vx_t$ (e.g. high school GPA, SAT scores, etc.) depend on both the agent's private type $\vu_t$ (e.g. a student's background) via initial features $\vb_t$ (e.g. the SAT score or HS GPA student $t$ would get without studying) and effort conversion matrix $\mathcal{E}_t$ (e.g. how much studying translates to an increase in SAT score for student $t$) and assessment rule $\vtheta_t$ via action $\va_t$, which could correspond to studying, taking an SAT prep course, etc). An agent's outcome $y_t$ (e.g. college GPA) is determined by their observable features $\vx_t$ (via causal relationship $\vtheta^*$) and type $\vu_t$ (via baseline outcome error term $o_t$, which could be lower for students from underserved groups due to institutional barriers, discrimination, etc).}
    \label{fig:intro-dags}
    
\end{figure*}

\xhdr{Latent variables: agent types.} 
Each agent $t$ has an unobserved private type $\vu_t$ that can impact both their observed attributes $\vx_t$ and true outcomes $y_t$. (The type is the confounder we would like to correct for.) In our running example, the type may broadly refer to the student's relevant background factors that cannot be directly observed or measured. For example, the student's type can specify their socioeconomic background factors (including the level of educational support they receive within their immediate family), as well as their interest and skills in specific subjects such as English or Mathematics.\footnote{Note that later in Section \ref{sec:experiment}, we use the terminology of agent \emph{subpopulations}. Subpopulations are distinct from types in that subpopulations determine the \emph{distribution} of types, but individual agents belonging to the same subpopulation may have different types. We will elaborate on this in Section \ref{sec:experiment}.}
Formally, we assume the type $\vu_t$ characterizes several relevant latent attributes of the agent, which we refer to using the tuple $\vu_t := (\vb_t, \mathcal{E}_t, o_t)$:

\begin{itemize}[noitemsep,nolistsep]
    \item $\vb_t \in \mathbb{R}^m$ specifies agent $t$'s baseline observable attribute values. For example, it can specify the baseline values of high school grades and SAT score the student would have received without any effort spent studying or preparing for standardized tests.
    \item $\mathcal{E}_t$ specifies agent $t$'s \emph{effort conversion matrix}---that is, how various effort investments in unobservable activities translate to changes in observable features. 
    \item $o_t$ summarizes all other environmental factors that can impact the agent's true outcome when we control for observable attributes. For example, it may reflect the effect of the institutional barriers the student faces on their actual college GPA.
\end{itemize}

We assume agent $t$'s observable features are affected by their type and effort investments. In particular, we assume they take the form $\mathbf{x}_t = \mathbf{b}_t + \mathcal{E}_t \mathbf{a}_t$.

\xhdr{Agent best responses.}
We assume the agent selects their effort profile $\mathbf{a}_t$ in order to maximize their predicted outcome $\hat{y}_t$, subject to some \emph{effort} cost $c(\cdot)$ associated with modifying their observable attributes. In particular, we assume the cost function is quadratic, that $c(\va_t) = \frac{1}{2} \|\va_t\|^2_2$. (Note that this assumption is common in the strategic learning literature; see, e.g., \cite{shavit2020strategic, mendler2020stochastic, dong2018strategic}). Formally, we assume agent $t$ selects their effort $\va_t$ by solving the following optimization problem: $\max_{\va} \left\{ \hat{y}_t -  \frac{1}{2} \|\mathbf{a}\|^2_2 \right\}$.
It is easy to see that for a given deployed assessment rule $\vtheta_t$, the agent's best-response effort investment is $\mathbf{a}_t = \mathcal{E}_t^\top \vtheta_t$.

\xhdr{True causal outcome model.} \emph{After} each round, the principal gets to observe the agent's true outcome $y_t \in \mathbb{R}$, which takes the form $y_t = \mathbf{x}_t^\top \vtheta^{*} + o_t$.
%
%
Here $\vtheta^*$ is the \emph{true causal} relationship between an agent's observable features and outcome. (Recall that $o_t \in \mathbb{R}$ captures the dependence of agent $t$'s outcome $y_t$ on unobservable or unmeasured factors.) 
We are interested in learning $\vtheta^* \in \mathbb{R}^m$, which can be interpreted as specifying how interventions impacting the value of $\vx$ lead to changes in $y$. Therefore, we say that an observable feature $x_i$ is \emph{causally relevant} if $\theta^*_i \neq 0$. For convenience, throughout we denote the subset of causally-relevant features by $\vx_\mathcal{C}$, where $\mathcal{C} \subseteq [n], \forall i \in \mathcal{C}$ if $\theta^*_i \neq 0$.

 \subsection{Overview of Results}
 %

 \xhdr{Strategic regression as instruments.}
Since $\vb_t$, $\mathcal{E}_t$, and $o_t$ may be correlated with one another, ordinary least squares generally will \emph{not} produce a consistent estimator for $\vtheta^*$ (see \Cref{sec:ols} for details). 
We make the novel observation that the principal's assessment rule $\vtheta_t$ is a valid \emph{instrument}, and leverage this observation to recover $\vtheta^*$ via Two-Stage Least Squares regression (2SLS). Our method applies to both \emph{off-policy} and \emph{on-policy} settings: one can directly apply 2SLS on historical data $\{(\vtheta_t, \vx_t, y_t) \}_{t=1}^T$, or the principal can intentionally deploy a sequence of varying assessment rules (e.g., by making small perturbations on a fixed rule) and then apply 2SLS on the collected data. 

Additionally, we show that our recovery of $\vtheta^*$ can be utilized to improve decision-making across several desired criteria, namely, individual fairness, agent outcomes, and predictive risk.

\xhdr{(Non-)causal assessment rules and fairness.}
In Section~\ref{sec:fairness}, we analyze the individual-level disparities that may result if the assessment rule deviates from $\vtheta^*$. Unlike most existing definitions of individual fairness, which rely on the observed characteristics of individuals, our definition measures the similarity between two individuals solely by comparing their $\vb$'s and $\mathcal{E}$'s---that is, we consider two individuals to be similar if they have the same baseline values for causally relevant observable features and similar potentials for improving these observable attributes through effort investments. Individual fairness then requires similar individuals to receive similar decisions. (We note that while our notion of individual fairness may not be easy to estimate using observational data, it is a more fine-grained---and arguably better justified---notion of individual fairness, as it distinguishes between the causally relevant and causally irrelevant facets of observable features.)  We show that when making predictions using $\vtheta = \vtheta^*$, our notion of individual fairness is satisfied, but when the assessment rule deviates from $\vtheta^*$, i.e., $\vtheta \neq \vtheta^*$, individual fairness may be violated by an arbitrarily large amount.


 %
 \xhdr{Agent outcome maximization.}
Note that a decision-maker can use the assessment rule $\vtheta$ as a form of intervention to incentivize agents to invest their efforts optimally toward maximizing their outcomes ($y$). In Section~\ref{sec:agent-outcome}, we show that utilizing the causal parameters recovered during our 2SLS procedure, one can find the assessment rule maximizing expected agent outcomes.
 
\xhdr{Predictive risk minimization.}
Another commonly-studied goal for decision-makers is \emph{predictive risk minimization}, 
which aims to minimize $\E[(\hat y_t - y_t)^2]$, the expected squared difference between an agent's \emph{true} outcome and the outcome predicted by the assessment. Compared to standard regression, this is a more challenging objective since both the prediction $\hat y_t$ and outcome $y_t$ depend on the deployed rule $\vtheta_t$. This leads to a non-convex risk function. In Section~\ref{sec:prm}, we show that the knowledge of $\vtheta^*$ enables us to compute an unbiased estimate of the gradient of the predictive risk. As a result, we can apply stochastic gradient descent to find a local minimum of predictive risk function. 

\xhdr{Empirical observations.} 
In Section~\ref{sec:experiment}, we empirically confirm and illustrate the performance of our algorithm. In particular, 
for a semi-synthetic dataset inspired by our university admissions example, we observe that our methods consistently estimate the true causal relationship between observable features and outcomes (at a rate of $\mathcal{O}(1/\sqrt{T})$), whereas OLS does not. Notably, OLS mistakenly estimates that SAT is causally related to college GPA, even though our experimental setup assumes it is not. On the other hand, our 2SLS-based method avoids this erroneous estimation. We also show that our methods outperform standard SGD methods in the predictive risk minimization setting.


\subsection{Related Work}

An active area of research on \emph{strategic learning} aims to develop machine learning algorithms that are capable of making accurate predictions about decision subjects even if they respond \emph{strategically} and potentially \emph{untruthfully} to the choice of the predictive model~ \citep{dong2018strategic,hardt2016strategic,  mendler2020stochastic, shavit2020strategic, grinder, CummingsIL15, CaiDP15, MEIR2012123, hu2018disparate}. Generalizing strategic learning, \citet{perdomo2020performative} propose a framework called \emph{performative predictions}, which broadly studies settings in which the act of predicting influences the prediction target. Several recent papers have investigated the relationship between strategic learning and causality~\citep{shavit2020strategic,BechavodLWZ21,miller2020strategic}.

The setting most similar to ours is that of \citet{shavit2020strategic}. They consider a strategic classification setting in which an agent's outcome is a linear function of features --some observable and some not (see \Cref{fig:shavit-model-dag} in the appendix for a graphical representation of their model). While they assume that an agent's latent attributes can be modified strategically, we choose to model the agent as having an unmodifiable private \emph{type}. Both of these assumptions are reasonable, and some domains may be better described by one model than the other. For example, the model of \citeauthor{shavit2020strategic} may be useful in a setting such as car insurance pricing, where some unobservable factors related to safe driving are modifiable. On the other hand, our model captures settings like university admissions, where confounding factors (e.g., socioeconomic background) are not easily modifiable. Both models are special cases of a broader causal graph (described in Appendix \ref{sec:shavit}). Note that in the model of \citeauthor{shavit2020strategic}, $\vtheta_t$ violates the backdoor criterion and therefore cannot serve as a valid instrument. \cite{BechavodLWZ21} consider a setting simpler than ours in which there are no confounding effects from agents’ unobserved types on their observable features and outcomes. As a result, the authors can apply standard least squares regression techniques to recover causal parameters.

Our work is also related to
\citet{miller2020strategic}, which shows that designing good incentives for agent improvement in strategic classification  is at least as hard as orienting edges in the corresponding causal graph. In contrast to their work, we make the observation that the assessment rule deployed by the principal can be actively used as a valid \emph{instrument}, which allows us to circumvent this hardness result by performing an \emph{intervention} on the causal graph of the model.

Instrumental variable (IV) regression \citep{angrist2001iv, angrist1995tsls,angrist1996iv} has mostly been used for observational studies (see e.g., \cite{angrist06, bloom}). Similar to ours, there is recent work on constructing instruments through dynamic action recommendations in multi-armed bandits settings \cite{ngo2021ivbandits,Kallus2018InstrumentArmedB}. 
We consider an orthogonal direction: constructing instruments through assessment rules in the strategic learning setting.

\section{IV Regression through Strategic Learning}\label{sec:main}

Instrumental variable (IV) regression  allows for consistent estimation of the relationship between an outcome and observable features in the presence of confounding terms. In this setting, we view the assessment rules $\{\vtheta_t\}_{t=1}^T$ as algorithmic instruments and perform IV regression to estimate the true causal relationship $\vtheta^*$. 
There are three criteria for $\vtheta_t$ to be a valid instrument: (1) $\vtheta_t$ influences the observable features $\vx_t$, (2) $\vtheta_t$ only influences the outcome $y_t$ through $\vx_t$, and (3) $\vtheta_t$ is independent from the private type $\vu_t$. By design, criterion (1) and (2) are satisfied. We aim to design a mechanism that satisfies criterion (3) by choosing the assessment rule $\vtheta_t$ independently of the private type $\vu_t$. As can be seen by \Cref{fig:intro-dags}, the principal's assessment rule $\vtheta_t$ satisfies these criteria.

We focus on two-stage least-squares regression (2SLS), a family of techniques for IV estimation.
Intuitively, 2SLS can be thought of as estimating the causal relationship $\vtheta^*$ between $\vx_t$ and $y_t$ by perturbing the instrument $\vtheta$ and measuring the change in $\vx_t$ and $y_t$. This enables us to account for the change in $y_t$ \emph{as a result of} the change in $\vx_t$. 2SLS does this by independently estimating the relationship between an \emph{instrumental variable} $\vtheta_t$ and the observable features $\vx_t$, as well as the relationship between $\vtheta_t$ and the outcome $y_t$ via simple least squares regression. For more background on the specific version of 2SLS we use, see Section 4.8 of \cite{cameron2005microeconometrics}.

 Formally, given a set of observations $\{\vtheta_t, \vx_t, y_t\}_{t=1}^T$, we compute the estimate $\widehat{\vtheta}$ of the true casual parameters $\vtheta^*$ from the following process of two-stage least squares regression (2SLS). We use $\widetilde{\vtheta}_t$ to denote the vector $\begin{bmatrix} \vtheta_t & 1 \end{bmatrix}^\top$.
\begin{enumerate}[noitemsep,nolistsep]
    \item Estimate $\Omega = \mathbb{E}[\mathcal{E}_t \mathcal{E}_t^\top]$, $\mathbb{E}[\vb_t^\top]$ using 

    \begin{equation*}
        \begin{bmatrix}\widehat{\Omega}\\ \Bar{\vb}^\top \end{bmatrix} = \left(\sum_{t=1}^T \widetilde{\vtheta}_t \widetilde{\vtheta}_t^\top \right)^{-1} \sum_{t=1}^T \widetilde{\vtheta}_t \vx_t^\top        
    \end{equation*}
    
    \item Estimate $\boldsymbol{\lambda} = \Omega \vtheta^*$, $(\mathbb{E}[o_t] + \mathbb{E}[\vb_t^\top] \vtheta^*)$ using 
    
    \begin{equation*}
        \begin{bmatrix}\widehat{\boldsymbol{\lambda}}\\ \Bar{o} + \Bar{\vb}^\top \vtheta^* \end{bmatrix} = \left(\sum_{t=1}^T \widetilde{\vtheta}_t \widetilde{\vtheta}_t^\top \right)^{-1} \sum_{t=1}^T \widetilde{\vtheta}_t y_t
    \end{equation*}

    \item Estimate $\vtheta^*$ as $\widehat{\vtheta} = \widehat{\Omega}^{-1} \widehat{\boldsymbol{\lambda}}$
\end{enumerate}

We assume that $\sum_{t=1}^T \widetilde{\vtheta}_t \widetilde{\vtheta}_t^\top$ is invertible, as is standard in the 2SLS literature. For proof that IV regression produces a consistent estimator of $\vtheta^*$ under our setting, see Appendix \ref{sec:iv-regression}.

\begin{theorem}\label{thm:theta-bound}
    Given a sequence of bounded assessment rules $\{\vtheta_t\}_{t=1}^T$ and the (observable feature, outcome) pairs $\{(\vx_t, y_t)\}_{t=1}^T$ they induce, the distance between the true causal relationship $\vtheta^*$ and the estimate $\widehat{\vtheta}$ obtained via IV regression is bounded as

    \begin{equation*}
        \| \widehat{\vtheta} - \vtheta^* \|_2 = \widetilde{\mathcal{O}}\left( \frac{ \sqrt{mT \log(1/\delta) }}{\sigma_{min}\left(\sum_{t=1}^T \vtheta_t (\vx_t - \Bar{\vb})^\top \right)} \right)
    \end{equation*}

    with probability $1 - \delta$, if $o_t$ is a bounded random variable.
\end{theorem}

\emph{Proof Sketch.} While similar bounds exist for traditional IV regression problems, they do not apply to the strategic learning setting we consider. See Appendix \ref{sec:theta-bound-proof} for the full proof. The bound follows by substituting our expressions for $\vx_t$, $y_t$ into the IV regression estimator, applying the Cauchy-Schwarz inequality to split the bound into two terms (one dependent on $\{(\vtheta_t, \vx_t)\}_{t=1}^T$ and one dependent on $\{(\vtheta_t, o_t)\}_{t=1}^T$), and using a Chernoff bound to bound the term dependent on $\{(\vtheta_t, o_t)\}_{t=1}^T$ with high probability.

While in some settings, the principal may only have access to \emph{observational} (e.g., batch) data, in other settings, the principal may be able to actively deploy assessment rules on the agent population. We show that in scenarios in which this is possible, the principal can play random assessment rules centered around some ``reasonable'' assessment rule to achieve an $\mathcal{O}\left(\frac{1}{\sigma_{\theta}^2 \sqrt{T}} \right)$ error bound on the estimated causal relationship $\widehat{\vtheta}$, where $\sigma_{\theta}^2$ is the variance in each coordinate of $\vtheta_t$. Note that while playing random assessment rules may be seen as unfair in some settings, the principal is free to set the variance parameter $\sigma_{\theta}^2$ to an ``acceptable'' amount for the domain they are working in. We formalize this notion in the following corollary.

\begin{corollary} \label{cor:theta-bound}
    If each $\theta_{t,j}$, $j \in {1, \ldots, m}$, is drawn independently from some distribution $\mathcal{P}_j$ with variance $\sigma_{\theta}^2$, $\vb_t$ and $\mathcal{E}_t$ are bounded random variables, $\mathcal{E}_t \mathcal{E}_t^\top$ is full-rank, and $\sigma_{min}(\mathbb{E}[\mathcal{E}_t \mathcal{E}_t^\top]) > 0$, then

    \begin{equation*}
        \| \widehat{\vtheta} - \vtheta^* \|_2 = \widetilde{\mathcal{O}} \left( \frac{\sqrt{m \log(1/\delta) }}{\sigma_{\theta}^2 \sqrt{T}} \right)
    \end{equation*}

    with probability $1 - \delta$.
\end{corollary}

\emph{Proof Sketch.} We begin by breaking up $\sigma_{min}\left(\sum_{t=1}^T \vtheta_t (\vx_t - \Bar{\vb})^\top \right)$ into two terms, $\|A\|_2$ and $\sigma_{min}(B)$, where $A$ and $B$ are functions of $\sum_{t=1}^T \vtheta_t (\vx_t - \Bar{\vb})^\top$. We use the Chernoff and matrix Chernoff inequalities to bound $\|A\|_2$ and $\sigma_{min}(B)$ with high probability respectively. For the full proof, see Appendix \ref{sec:denom}.

\section{(Un)fairness of (Non-)causal Assessments}\label{sec:fairness}

While making predictions based on causal relationships is important from an ML perspective for reasons of generalization and robustness, the societal implications of using non-causal relationships to make decisions are perhaps an even more persuasive reason to use causally-relevant assessments. 
In particular, it could be the case that a certain individual is worse at strategically manipulating features which are not causally relevant when compared to their peers. If these attributes are used in the decision-making process, this agent may be unfairly seen by the decision-maker as less qualified than their peers, even if their initial features and ability for improvement is similar to others. 

One important criterion for assessing the fairness of a machine learning model at the individual level is that two individuals who have similar merit should receive similar predictions. \citet{dwork2012fairness} formalize this intuition through the notion of \emph{individual fairness}, which is formally defined as follows. 
\begin{definition}[Individual Fairness \cite{dwork2012fairness}]
A mapping $M: \mathcal{U} \rightarrow \Delta (Y)$ is individually fair if for every $\vu, \vu' \in \mathcal{U}$, we have 
\begin{equation*}
    D(M(\vu), M(\vu')) \leq d(\vu, \vu'),
\end{equation*}
where $\vu, \vu' \in \mathcal{U}$ are individuals in population $\mathcal{U}$, $\Delta (Y)$ is the probability distribution over predictions $Y$, $D(M(\vu), M(\vu'))$ is a distance function which measures the similarity of the predictions received by $\vu$ and $\vu'$, and $d(\vu, \vu')$ is a distance function which measures the similarity of the two individuals.
\end{definition}

Recall that in the setting we consider, the mapping between individuals and predictions is defined to be $M(\vu) := \vx^\top \vtheta + \hat{o} = (\vb + \mathcal{E}\mathcal{E}^\top)^\top \vtheta + \hat{o}$. The prediction an individual receives is deterministic, so a natural choice for $D(M(\vu), M(\vu'))$ is $|\hat{y} - \hat{y}'|$. We take a causal perspective when defining a metric $d(\vu, \vu')$ to measure the similarity of two individuals $\vu$ and $\vu'$. Intuitively, individuals that have similar initial causally-relevant features and ability to modify causally-relevant features should be treated similarly. Therefore, we define $d(\vu, \vu')$ to reflect the difference in causally-relevant components of $\vb$ \& $\vb'$ (initial feature values) and $\mathcal{E}\mathcal{E}^\top$ \& $\mathcal{E}'\mathcal{E}^{'\top}$ (ability to manipulate features). 
With this in mind, we are now ready to define the criterion for individual fairness to be satisfied in the strategic learning setting.

\begin{definition}\label{def:IF}
In the strategic learning setting, individual fairness is satisfied if
\small
\begin{equation*}
\begin{aligned}
    |\hat{y} - \hat{y}'| &\leq d(\vu, \vu')\\ &= \|\vb_\mathcal{C} - \vb'_\mathcal{C}\|_2 + \|(\mathcal{E}\mathcal{E}^\top)_\mathcal{C} - (\mathcal{E}'\mathcal{E}^{'\top})_\mathcal{C}\|_2,
\end{aligned}
\end{equation*}
\normalsize
where 
\small
\begin{align*}
    b_{\mathcal{C},i} &=  
    \begin{cases}
        b_i & \text{if } i \in \mathcal{C}\\
        0 & \text{otherwise}
    \end{cases}
    , &\\
    (\mathcal{E}\mathcal{E}^\top)_{\mathcal{C},ij} &=  
    \begin{cases}
        (\mathcal{E}\mathcal{E}^\top)_{ij} & \text{if } i \in \mathcal{C} \text{ or } j \in \mathcal{C}\\
        0 & \text{otherwise}.
    \end{cases}    
\end{align*}
\normalsize
\end{definition}
Recall that $\mathcal{C} \subseteq \{1, \ldots, n\}$ denotes the set of indices of observable features $\vx$ which are causally relevant to $y$ (i.e., $\theta^*_i \neq 0$ for $i \in \mathcal{C}$).

\begin{theorem}\label{thm:IF}
    Assessment $\vtheta = \vtheta^*$ satisfies individual fairness for any two agents $\vu$ and $\vu'$.
\end{theorem}

\emph{Proof Sketch.} See Appendix \ref{sec:fairness-derivations} for the full proof, which follows straightforwardly from the Cauchy-Schwarz inequality and the definition of the matrix operator norm. (Our results are not dependent upon the specific matrix or vector norms used, analogous results will hold for other popular choices of norm.) Throughout the proof we assume that $\|\vtheta^*\|_2 = 1$ by definition, although our results hold up to constant multiplicative factors if this is not the case.

While $\vtheta = \vtheta^*$ satisfies the criterion for individual fairness, this will generally not be the case for an arbitrary assessment $\vtheta \neq \vtheta^*$. For instance, consider the case where $d(\vu, \vu') = 0$ for two agents $\vu$ and $\vu'$. Under this setting, it is possible to express $|\hat{y} - \hat{y}'|$ using quantities which do not depend on $d(\vu, \vu')$. As these quantities increase, $|\hat{y} - \hat{y}'|$ increases as well, despite the fact that $d(\vu, \vu')$ remains constant.
\begin{theorem}\label{thm:IF-bound}
    For any deployed assessment rule $\vtheta$, the gap in predictions between two agents $\vu$ and $\vu'$ such that $d(\vu, \vu') = 0$ is 
    
    \begin{equation*}
    \begin{aligned}
        |\hat{y} - \hat{y}'| = &\left| \sum_{i \not \in \mathcal{C}} (b_i - b_i') \theta_i \right. \\ &\left.+ \sum_{i \not \in \mathcal{C}} \sum_{j \not \in \mathcal{C}} ((\mathcal{E}\mathcal{E}^\top)_{ij} - (\mathcal{E}'\mathcal{E}^{'\top})_{ij}) \theta_i \theta_j \right|\\
    \end{aligned}
    \end{equation*}
    
\end{theorem}

See Appendix \ref{sec:fairness-derivations} for the full derivation. Note that all components of $\vtheta$ which appear in Theorem \ref{thm:IF-bound} are outside of the support of $\vtheta^*$. 

In order to illustrate how $|\hat{y} - \hat{y}'|$ can grow while $d(\vu, \vu')$ remains constant, consider the following example.

\begin{example}\label{ex:fairness}
Consider a setting in which the distance $d(\vu, \vu') = 0$ between agents $\vu$ and $\vu'$, and there is a one-to-one mapping between actions and observable features for each agent, with one agent having an advantage when it comes to manipulating features which are not causally relevant. Formally, let $\vtheta^* = [\mathbf{0}_{n/2}^\top \;\; \sqrt{\frac{2}{n}} \mathbf{1}_{n/2}^\top]^\top$, $\vb = \vb'$, $\mathcal{E} = \boldsymbol{\delta} I_{n \times n}$, and $\mathcal{E}' = \boldsymbol{\delta}' I_{n \times n}$, where $\boldsymbol{\delta} = [\sqrt{n} \mathbf{1}_{n/2}^\top \;\; \mathbf{1}_{n/2}^\top]^\top$ and $\boldsymbol{\delta}' = [\mathbf{0}_{n/2}^\top \;\; \mathbf{1}_{n/2}^\top]^\top$.

Under such a setting, the equation in Theorem \ref{thm:IF-bound} simplifies to

\begin{equation*}
    |\hat{y} - \hat{y}'| = n \sum_{i \not \in \mathcal{C}} \theta_i^2 = n \sum_{i = 1}^{n/2} \theta_i^2.
\end{equation*}

For the full derivation, see Appendix \ref{sec:fairness-derivations}. Suppose now that the assessment $\vtheta$ puts weight at least $1/\sqrt{n}$ on each observable feature which is not causally relevant. Under such a setting, $|\hat{y} - \hat{y}'| \geq n/2$, meaning that the difference in predictions tends towards infinity as $n$ grows large, despite the fact that $d(\vu, \vu') = 0$ and $y = y'$!
\end{example}

\section{Agent Outcome Improvements}\label{sec:agent-outcome}

In the strategic learning setting, the goal of each agent is clear: they aim to achieve the highest prediction $\hat{y}$ possible, regardless of their true label $y$. On the other hand, what the goal should be for the principal is less clear, and depends on the specific setting being considered. For example, in some settings it may be enough to discover the causal relationships between observable features and outcomes. However in other settings, the principal may wish to take a more active role. In particular, when making decisions which have real-world consequences, it may be in the principal's best interest to use a decision rule which promotes desirable behavior \cite{kleinberg2020classifiers, shavit2020strategic, harris2021stateful}, i.e., behavior which has the potential to improve the \emph{actual outcome} of an agent.


In the agent outcome improvement setting, the goal of the principal is to maximize the expected outcome $\mathbb{E}[y]$ of an agent drawn from the agent population. In our college admissions example, this would correspond to deploying an assessment rule with the goal of \emph{maximizing} expected student college GPA. Formally, we aim to find $\vtheta^{AO}$ in a convex set $\mathcal{S}$ of feasible assessment rules such that the induced expected agent outcome $\mathbb{E}[y]$ is maximized.

After some algebraic manipulation, the optimization becomes $\vtheta^{AO} = \arg \max_{\vtheta \in \mathcal{S}} \vtheta^\top \boldsymbol{\lambda},$
where $\boldsymbol{\lambda} = \mathbb{E}[\mathcal{E}_t \mathcal{E}_t^\top]\vtheta^*$.

For the full derivation, see \Cref{sec:AO-derivation}. Note that while the principal never directly observes $\mathbb{E}[\mathcal{E}_t \mathcal{E}_t^\top]$ nor $\vtheta^*$, they estimate $\boldsymbol{\lambda} = \mathbb{E}[\mathcal{E}_t \mathcal{E}_t^\top]\vtheta^*$ during the second stage of the 2SLS procedure. Therefore, if the principal has already run 2SLS to recover a sufficiently accurate estimate of the causal parameters $\vtheta^*$, they can estimate the agent outcome-maximizing decision rule by solving the above optimization.

\section{Experiments}\label{sec:experiment}

We empirically evaluate our model on a semi-synthetic dataset inspired by our running university admissions example. We compare our 2SLS-based method against ordinary least squares (OLS), which directly regresses observed outcomes $y$ on observable features $\mathbf{x}$. We show that even in our stylized setting with just two observable features, OLS does not recover $\vtheta^*$, whereas our method does.

\begin{figure}
    \centering
    \includegraphics[width=0.5\linewidth]{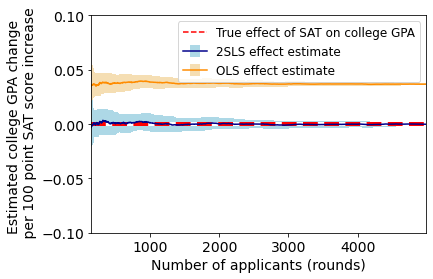}
    \caption{OLS versus 2SLS estimates for SAT effect on college GPA over 5000 rounds. Results are averaged over 10 runs, with the error bars (in lighter colors) representing one standard deviation. The red dashed line is the true causal relationship between SAT score and college GPA.}
    \label{fig:sat-estimates}
\end{figure}

\begin{figure}[ht]
    \centering
    \includegraphics[width=0.5\linewidth]{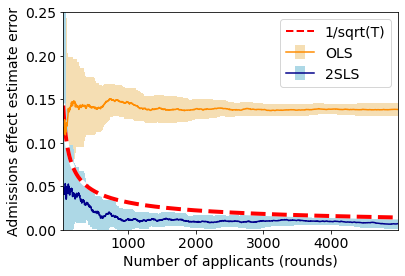}
    \caption{OLS effect estimate error $\|\widehat{\vtheta}_{\text{OLS}} - \vtheta^* \|_2$ (in orange) and 2SLS estimate error $\|\widehat{\vtheta}_{\text{2SLS}} - \vtheta^* \|_2$ (in blue) over 5000 rounds. Results are averaged over 10 runs. Error bars (in lighter colors) represent one standard deviation. 2SLS estimate error decreases at a rate of about $\frac{1}{\sqrt{T}}$ (red dashed line). }
    \label{fig:estimation_error}
\end{figure}

\xhdr{University admissions experimental description} We constructed a semi-synthetic dataset based on the SATGPA dataset, a collection of real university admissions data.\footnote{Originally collected by the Educational Testing Service, the SATGPA dataset is publicly available and can be found here: \url{https://www.openintro.org/data/index.php?data=satgpa}.} The SATGPA dataset contains 6 variables on 1000 students. We use the following: two features (high school (HS) GPA and SAT score) and an outcome (college GPA). Using OLS (which is assumed to be consistent since we have yet to modify the data to include confounding), we find that the effect of $[\text{SAT}, \text{HS GPA}]$ on college GPA in this dataset is $\vtheta^*=[0.0015, 0.5895]^\top$. We then construct synthetic data that is based on this original data, yet incorporates confounding factors. For simplicity, we let the true effect $\vtheta^*=[0, 0.5]^\top$. That is, we assume HS GPA is causally related to college GPA, but SAT score is not.\footnote{Though this assumption may be contentious, it is based on existing research (e.g.,  \citet{allensworth2020gpasat}).} We consider two private types of applicant backgrounds: \textit{disadvantaged} and \textit{advantaged}. Disadvantaged applicants have lower initial HS GPA and SAT ($\vb$), lower baseline college GPA ($o$), and need more effort to improve observable features ($\mathcal{E}$).\footnote{For example, this could be due to the disadvantaged group being systemically underserved or marginalized (and the converse for advantaged group).} Each applicants' initial features are randomly drawn from one of two Gaussian distributions, depending on background. Applicants may manipulate both of their features. See \Cref{sec:omitted-experiments} for a full experimental description.


\xhdr{Results.} In \Cref{fig:sat-estimates}, we compare the true effect of SAT score on college GPA ($\vtheta^*$) with the estimates of these quantities given by our method of 2SLS from \Cref{sec:main} ($\hat{\vtheta}_{\text{2SLS}}$) and with the estimates given by OLS ($\hat{\vtheta}_{\text{OLS}}$). (An analogous figure for the effects of HS GPA is included in the appendix.) In \Cref{fig:estimation_error}, we compare the estimation errors of OLS and 2SLS, i.e.   $\|\widehat{\vtheta}_{\text{OLS}} - \vtheta^* \|_2$ and $\|\widehat{\vtheta}_{\text{2SLS}} - \vtheta^* \|_2$. 

We find that our 2SLS method converges to the true causal relationship (at a rate of about $\frac{1}{\sqrt{T}}$), whereas OLS has a constant bias. Although our setting assumes that SAT score has no causal relationship with college GPA, OLS mistakenly predicts that, on average, a 100 point increase in SAT score leads to about a 0.05 point increase in college GPA. If SAT were not causally related to collegiate performance in real life, these biased estimates could lead universities to erroneously use SAT scores in admissions decisions. This highlights the advantage of our method, since using a naive parameter estimation method like OLS in the presence of confounding could cause decision-making institutions to deploy assessments which don't accurately reflect the characteristics they are trying to measure.

\subsection{Predictive Risk Minimization}\label{sec:prm}

Analogous to recovering causal relationships and improving agent outcomes, another common goal of the principal in the strategic learning setting is to \emph{minimize predictive risk}. Formally, the goal of the principal in the predictive risk minimization setting is to learn the assessment rule that minimizes the expected squared difference between an agent's true outcome and the outcome predicted by the principal, i.e., $f(\vtheta_t) = \mathbb{E}[(\widehat{y}_t - y_t)^2]$.

 Due to the dependence of $\vx_t$ and $y_t$ on $\vtheta_t$, $f(\vtheta_t)$ will be non-convex in general, and can have several extrema which are not global minima, even in the case of just one observable feature. When faced with such non-convex optimization problems, gradient descent is often a popular approach due to its simplicity and convergence to local minima in practice.

If the effort conversion matrix $\mathcal{E}$ is the same for all agents, the gradient of population risk function can be written as

\begin{equation*}
    \nabla_{\vtheta_t} f(\vtheta_t) = 2(\mathbb{E}[(\widehat{y}_t - y_t) \vx_t] + \mathbb{E}[\widehat{y}_t - y_t]\mathcal{E}\mathcal{E}^\top (\vtheta_t - \vtheta^*).
\end{equation*}

See Appendix \ref{sec:pop-grad-derivation} for the derivation. In our college admissions example, this would correspond to the setting in which all students' math GPA, SAT scores, etc. improve the same amount given the same effort: this may be a reasonable assumption if the students being considered have the same ability to learn, despite other differences in background they may have. If $\mathcal{E}\mathcal{E}^\top$ is known to the principal (e.g. through the 2SLS procedure in Section \ref{sec:main}), then each $(\vtheta_t, \vx_t, y_t)$ tuple can be used to compute an unbiased estimate of $\nabla_{\vtheta_t} f(\vtheta_t)$ for use in online gradient descent. 

Recent work on \textit{performative prediction} \cite{perdomo2020performative, mendler2020stochastic, miller2021outside} examines the use of repeated gradient descent in the strategic learning setting and finds that repeated gradient descent generally converges to \emph{performatively stable} points. There is no direct comparison between performatively stable points and local minima in our setting. In fact, performatively stable points can actually \emph{maximize} predictive risk under some settings. (See \citet{miller2021outside} for such an example.) Our methods differ from this line of work because we take $\vx_t$, $y_t$, and $\widehat{y}_t$'s \emph{direct dependence} on the assessment rule $\vtheta_t$ into account when calculating the gradient of the risk function, whereas these performative prediction models (henceforth \textit{simple stochastic gradient descent} or \textit{SSGD}) do not. While SSGD may be satisfactory for some settings, it produces a biased estimate of the gradient in general, which can lead to unexpected behavior under our setting; by contrast, our gradient estimate is unbiased (see \Cref{fig:pp-vs-us}). Even in situations which SSGD \emph{does} get the sign of the gradient correct, it may converge at a much slower rate, due to its incomplete estimate of the gradient (see \Cref{fig:pp-vs-us-convergence-rate} in Appendix \ref{sec:figure-setting}).

\begin{figure}
    \centering
    \includegraphics[width=0.5\linewidth]{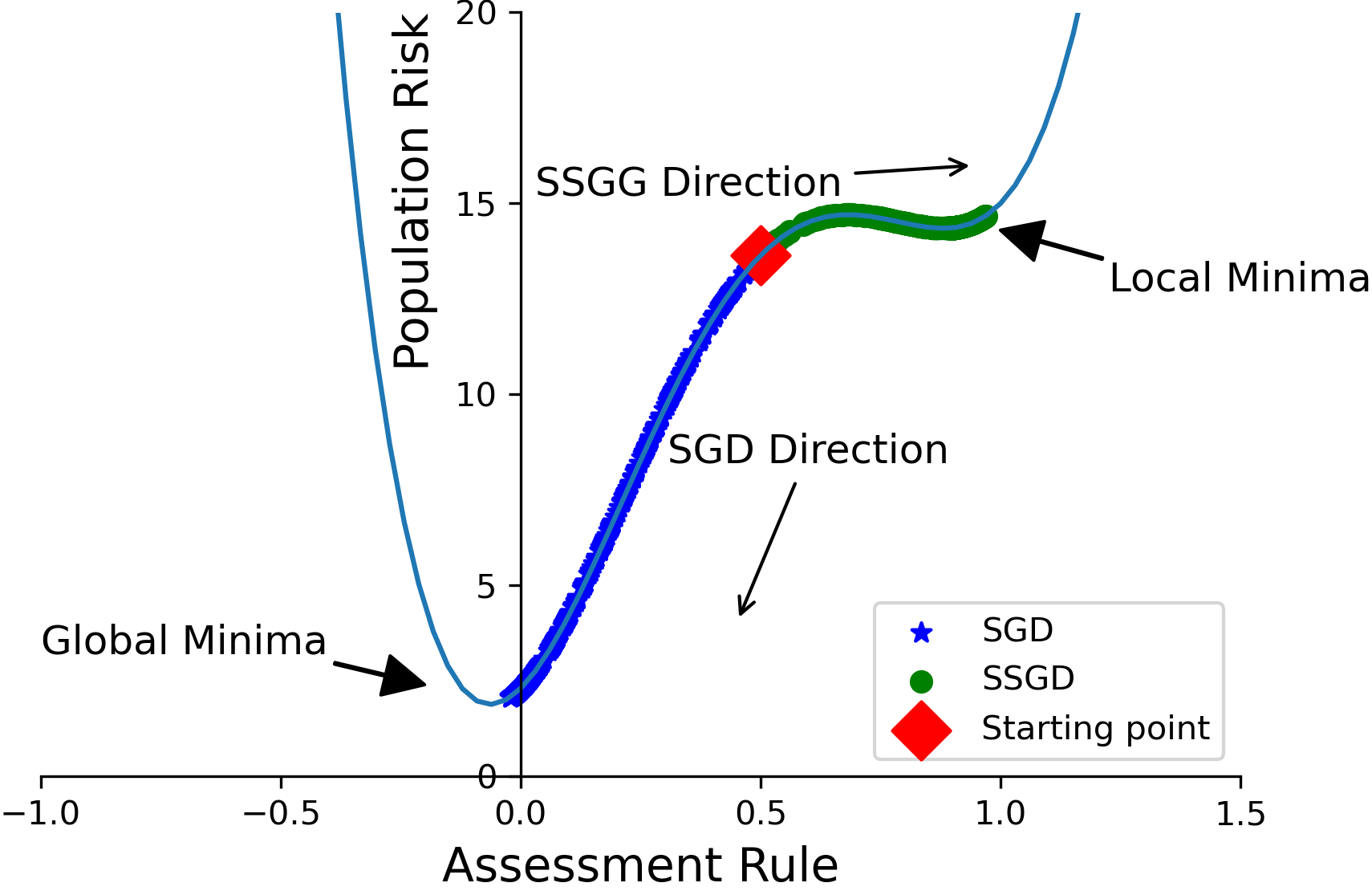}
    \caption{Stochastic Gradient Descent (SGD, takes into account $x_t$, $y_t$, $\widehat{y}_t$'s dependence on $\theta_t$) vs Simple Stochastic Gradient Descent (SSGD, does not). In the 1D setting, it is possible for the gradient of SSGD to have the wrong sign. When both are initialized at $\theta_0 = 0.5$, SGD is able to follow the gradient and converge to the global minima, while SSGD is not. We ran each method for $1000$ time-steps with a decaying learning rate of $\frac{0.001}{\sqrt{T}}$.}
    \label{fig:pp-vs-us}
\end{figure}

\section{Conclusion}

In this work, we establish the possibility of recovering the causal relationship between observable attributes and the outcome of interest in settings where a decision-maker utilizes a series of linear assessment rules to evaluate strategic individuals. Our key observation is that in strategic settings, assessment rules serve as valid instruments (because they causally impact observable attributes but do not directly affect the outcome). This observation enables us to present a 2SLS method to correct for confounding bias in causal estimates. 
We then demonstrate the potential of the recovered causal coefficients to be utilized for preventing individual-level disparities, improving agent outcomes, and reducing predictive risk minimization.

%
While our work offers an initial step toward extracting causal knowledge from automated assessment rules, we rely on several simplifying assumptions---all of which mark critical directions for future work. In particular, we assume all assessment rules and the underlying causal model are linear. This assumption allows us to utilize linear IV methods. Extending our work to \emph{non-linear} assessment rules and IV methods is necessary for the applicability of our method to real-world settings. Another critical assumption is the agent's \emph{full knowledge} of the assessment rule and their \emph{rational} response to it, subject to a \emph{quadratic effort cost}. While these are standard assumptions in economic modeling, they need to be empirically verified in the particular decision-making context at hand before our method's outputs can be viewed as reliable estimates of causal relationships.

\section{Acknowledgements}
ZSW, KH, DN, LS were supported in part by the NSF FAI Award \#1939606, NSF SCC Award \#1952085, a Google Faculty Research Award, a J.P. Morgan Faculty Award, a Facebook Research Award, and a Mozilla Research Grant. HH acknowledges support from NSF (\#IIS2040929), J.P. Morgan, CyLab, Meta, and PwC. Any opinions, findings, conclusions, or recommendations expressed in this material are those of the authors and do not necessarily reflect the views of the National Science Foundation and other funding agencies.
Finally, the authors would like to thank Yonadav Shavit and John Miller for insightful conversations about \citet{shavit2020strategic} and \citet{miller2021outside}, respectively.

\bibliography{refs}
\bibliographystyle{icml2022}

\newpage
\appendix
\section{Parameter estimation in the causal setting}\label{sec:IV}

\subsection{Ordinary least squares is not consistent} \label{sec:ols}

The least-squares estimate of $\vtheta^*$ is given as 
\begin{equation*}
    \widehat{\vtheta}_{LS} = \left(\sum_{t=1}^T \vx_t \vx_t^\top \right)^{-1} \sum_{t=1}^T \vx_t y_t.
\end{equation*}

However, $\widehat{\vtheta}_{LS}$ is not a consistent estimator of $\vtheta^*$. To see this, let us plug in our expression for $y_t$ into our expression for $\widehat{\vtheta}_{LS}$. We get
\begin{equation*}
\widehat{\vtheta}_{LS} = \left(\sum_{t=1}^T \vx_t \vx_t^\top \right)^{-1} \sum_{t=1}^T \vx_t (\vx_t^\top \vtheta^* + o_t)
\end{equation*}

After distributing terms and simplifying, we get 

\begin{equation*}
    \widehat{\vtheta}_{LS} = \vtheta^* + \left(\sum_{t=1}^T \vx_t \vx_t^\top \right)^{-1} \sum_{t=1}^T \vx_t o_t.
\end{equation*}

$\vx_t$ and $o_t$ are not independent due to their shared dependence on the agent's private type $u_t$. Because of this, $\left(\sum_{t=1}^T \vx_t \vx_t^\top \right)^{-1}\sum_{t=1}^T \vx_t o_t$ will generally not equal $\mathbf{0}_m$, even as the number of data points (agents) grows large. To see this, recall that $\vx_t = \vb_t + \mathcal{E}_t \va_t$, so $\sum_{t=1}^T \vx_t o_t = \sum_{t=1}^T (\vb_t + \mathcal{E}_t \mathcal{E}_t^\top \vtheta_t) o_t$. $o_t$ and $\vb_t$ are both determined by the agent's private type. Take the example where $\vb_t = [o_t, 0, \ldots, 0]^\top$. In this setting, $\sum_{t=1}^T \vb_t o_t = [o_t^2, 0, \ldots, 0]^\top$, which will always be greater than 0 unless $o_t = 0$, $\forall t$.

\subsection{2SLS derivations} \label{sec:2SLS-derivations}
Define $\widetilde{\vtheta}_t = \begin{bmatrix} \vtheta_t\\ 1\end{bmatrix}$. $\vx_t$ can now be written as $\vx_t = \begin{bmatrix} \mathcal{E}_t \mathcal{E}_t^\top & \vb_t\end{bmatrix} \begin{bmatrix} \vtheta_t\\ 1 \end{bmatrix}$. 
\begin{lemma}
    Using OLS, we can estimate $\begin{bmatrix}\mathbb{E}[\mathcal{E}_t \mathcal{E}_t^\top]\\ \mathbb{E}[\vb_t]^\top \end{bmatrix}$ as 
    
    \begin{equation*}
    \begin{aligned}
        \begin{bmatrix}\widehat{\Omega}\\ \Bar{\vb}^\top \end{bmatrix} &= \left(\sum_{t=1}^T \widetilde{\vtheta}_t \widetilde{\vtheta}_t^\top \right)^{-1} \sum_{t=1}^T \widetilde{\vtheta}_t \vx_t^\top\\ &= \left(\sum_{t=1}^T \widetilde{\vtheta}_t \widetilde{\vtheta}_t^\top \right)^{-1} \begin{bmatrix} \sum_{t=1}^T \vtheta_t \vx_t^\top \\ \sum_{t=1}^T \vx_t^\top \end{bmatrix},
    \end{aligned}
    \end{equation*}
    
    where $\widehat{\Omega} = \left(\sum_{t=1}^T \vtheta_t \vtheta_t^\top \right)^{-1} \sum_{t=1}^T \vtheta_t (\vx_t - \Bar{\vb})^\top$.
\end{lemma}
\begin{proof}
In order to calculate $\widehat{\Omega}$, we will make use of the following fact:
\begin{fact}[Block Matrix Inversion (\cite{bernstein2009matrix})]
If a matrix $P$ is partitioned into four blocks, it can be inverted blockwise as follows:
\begin{equation*}
\begin{aligned}
        P &= \begin{bmatrix} A & B\\ C & D\end{bmatrix}^{-1}\\ &= \begin{bmatrix} A^{-1} + A^{-1}BE^{-1} CA^{-1} & -A^{-1}BE^{-1}\\ -E^{-1} CA^{-1} & E^{-1}\end{bmatrix},
\end{aligned}
\end{equation*}
where A and D are square matrices of arbitrary size, and B and C are conformable for partitioning. Furthermore, A and the Schur complement of A in P ($E = D-CA^{-1}B$) must be invertible.
\end{fact}

Let $A = \sum_{t=1}^T \vtheta_t \vtheta_t^\top$, $B = \sum_{t=1}^T \vtheta_t$, $C = \sum_{t=1}^T \vtheta_t^\top$, and $D = \sum_{t=1}^T 1 = T$. Note that $A$ is invertible by assumption and $E$ is a scalar, so is trivially invertible unless $CA^{-1}B = T$.

Using this formulation, observe that
\begin{equation*}
    \Bar{\vb}^\top  = -E^{-1} CA^{-1}\sum_{t=1}^T \vtheta_t \vx_t^\top + E^{-1} \sum_{t=1}^T \vx_t^\top
\end{equation*} and
\begin{equation*}
\begin{aligned}
    \widehat{\Omega} &= A^{-1}\sum_{t=1}^T \vtheta_t \vx_t^\top + A^{-1}BE^{-1} CA^{-1}\sum_{t=1}^T \vtheta_t \vx_t^\top\\ &- A^{-1}BE^{-1} \sum_{t=1}^T \vx_t^\top\\
\end{aligned}
\end{equation*}
Rearranging terms, we see that $\widehat{\boldsymbol{\lambda}}$ can be written as 
\begin{equation*}
\begin{aligned}
    \widehat{\Omega} &= A^{-1}\sum_{t=1}^T \vtheta_t \vx_t^\top\\ &+ A^{-1}B (E^{-1} CA^{-1}\sum_{t=1}^T \vtheta_t \vx_t^\top - E^{-1} \sum_{t=1}^T \vx_t^\top)\\
    &= A^{-1}\sum_{t=1}^T \vtheta_t \vx_t^\top - A^{-1}B \Bar{\vb}^\top
\end{aligned}
\end{equation*}

Finally, plugging in for $A$ and $B$, we see that 
\begin{equation*}
\begin{aligned}
    \widehat{\Omega} &= \left(\sum_{t=1}^T \vtheta_t \vtheta_t^\top \right)^{-1}\sum_{t=1}^T \vtheta_t \vx_t^\top\\ &- \left(\sum_{t=1}^T \vtheta_t \vtheta_t^\top \right)^{-1}\sum_{t=1}^T \vtheta_t \Bar{\vb}^\top\\
    &= \left(\sum_{t=1}^T \vtheta_t \vtheta_t^\top \right)^{-1}\sum_{t=1}^T \vtheta_t (\vx_t - \Bar{z})^\top
\end{aligned}
\end{equation*}
\end{proof}

Similarly, we can write $y_t$ as $y_t = \begin{bmatrix}\vtheta_t^\top & 1 \end{bmatrix} \begin{bmatrix}\mathcal{E}_t \mathcal{E}_t^\top \vtheta^*\\ o_t + \vb_t^\top \vtheta^* \end{bmatrix}$. 

\begin{lemma}
    Using OLS, we can estimate $\begin{bmatrix} \mathbb{E}[\mathcal{E}_t \mathcal{E}_t^\top] \vtheta^*\\ \mathbb{E}[o_t] + \mathbb{E}[\vb_t^\top] \vtheta^* \end{bmatrix}$ as 
    \begin{equation*}
    \begin{aligned}
        \begin{bmatrix}\widehat{\boldsymbol{\lambda}}\\ \Bar{o} + \Bar{\vb}^\top \vtheta^* \end{bmatrix} &= \left(\sum_{t=1}^T \widetilde{\vtheta}_t \widetilde{\vtheta}_t^\top \right)^{-1} \sum_{t=1}^T \widetilde{\vtheta}_t y_t\\ &= \left(\sum_{t=1}^T \widetilde{\vtheta}_t \widetilde{\vtheta}_t^\top \right)^{-1} \begin{bmatrix} \sum_{t=1}^T \vtheta_t y_t^\top \\ \sum_{t=1}^T y_t^\top \end{bmatrix},
    \end{aligned}
    \end{equation*}
    where $\widehat{\boldsymbol{\lambda}} = \left(\sum_{t=1}^T \vtheta_t \vtheta_t^\top \right)^{-1} \sum_{t=1}^T \vtheta_t (y_t - \Bar{o} - \Bar{\vb}^\top \vtheta^*)$.
\end{lemma}
\begin{proof}
The proof follows similarly to the proof of the previous lemma. Let $A = \sum_{t=1}^T \vtheta_t \vtheta_t^\top$, $B = \sum_{t=1}^T \vtheta_t$, $C = \sum_{t=1}^T \vtheta_t^\top$, and $D = \sum_{t=1}^T 1 = T$. Note that $A$ is invertible by assumption and $E$ is a scalar, so is trivially invertible unless $CA^{-1}B = T$.

Using this formulation, observe that
\begin{equation*}
    \Bar{o}^\top + \Bar{z}^\top \vtheta^*  = -E^{-1} CA^{-1}\sum_{t=1}^T \vtheta_t y_t + E^{-1} \sum_{t=1}^T y_t
\end{equation*} and
\begin{equation*}
\begin{aligned}
    \widehat{\boldsymbol{\lambda}} &= A^{-1}\sum_{t=1}^T \vtheta_t y_t\\ &+ A^{-1}B \left(E^{-1} CA^{-1}\sum_{t=1}^T \vtheta_t y_t - E^{-1} \sum_{t=1}^T y_t \right)\\
    &= A^{-1}\sum_{t=1}^T \vtheta_t y_t - A^{-1}B \left(\Bar{o}^\top + \Bar{z}^\top \vtheta^* \right)\\
    &= \left(\sum_{t=1}^T \vtheta_t \vtheta_t^\top \right)^{-1} \sum_{t=1}^T \vtheta_t \left(y_t - \Bar{o}^\top - \Bar{z}^\top \vtheta^* \right)
\end{aligned}
\end{equation*}
\end{proof}

\begin{theorem}
    We can estimate $\vtheta^*$ as 
    \begin{equation*}
        \widehat{\vtheta} = \widehat{\Omega}^{-1} \widehat{\boldsymbol{\lambda}} = \left(\sum_{t=1}^T \vtheta_t (\vx_t - \Bar{\vb})^\top \right)^{-1} \sum_{t=1}^T \vtheta_t (y_t - \Bar{o} - \Bar{z}^\top \vtheta^*)
    \end{equation*}
\end{theorem}
\begin{proof}
This follows immediately from the previous two lemmas.
\end{proof}

\subsection{2SLS is consistent} \label{sec:iv-regression}
Consider the two-stage least squares (2SLS) estimate of $\vtheta^*$,
\begin{equation*}
    \widehat{\vtheta}_{IV} = \left(\sum_{t=1}^T \vtheta_t (\vx_t - \Bar{\vb})^\top \right)^{-1} \sum_{t=1}^T \vtheta_t (y_t - \Bar{o} - \Bar{z}^\top \vtheta^*)
\end{equation*}

Plugging in for $y_t$ and simplifying, we get
\begin{equation*}
    \widehat{\vtheta}_{IV} = \vtheta^* + \left(\sum_{t=1}^T \vtheta_t (\vx_t - \Bar{\vb})^\top \right)^{-1} \sum_{t=1}^T \vtheta_t (o_t - \Bar{o})
\end{equation*}

To see that $\widehat{\vtheta}_{IV}$ \emph{is} a consistent estimator of $\vtheta^*$, we show that $\lim_{T \to \infty} \mathbb{E} \|\widehat{\vtheta}_{IV} - \vtheta^*\|_2^2 = 0$.

\begin{equation*}
    \mathbb{E} \|\widehat{\vtheta}_{IV} - \vtheta^*\|_2^2 = \mathbb{E} \left \|  \left(\sum_{t=1}^T \vtheta_t (\vx_t - \Bar{\vb})^\top \right)^{-1} \sum_{t=1}^T \vtheta_t (o_t - \Bar{o}) \right \|_2^2
\end{equation*}

$o_t - \Bar{o}$ and $\vtheta_t$ are uncorrelated, so $\sum_{t=1}^T \vtheta_t (o_t - \Bar{o})$ will go to zero as $T \rightarrow \infty$. On the other hand, $\sum_{t=1}^T \vtheta_t (\vx_t - \Bar{\vb})^\top$ will approach $T \mathbb{E}[\vtheta_t (\vx_t - \Bar{\vb})^\top]$. $\vtheta_t$ and $\vx_t - \Bar{\vb}$ \emph{are} correlated, so $\mathbb{E}[\vtheta_t (\vx_t - \Bar{\vb})^\top] \neq \mathbf{0}$ in general.
\section{Causal parameter recovery derivations}
\subsection{Proof of Theorem \ref{thm:theta-bound}} \label{sec:theta-bound-proof}

Recall that $\widehat{\vtheta} = \left(\sum_{t=1}^T \vtheta_t (\vx_t - \Bar{\vb})^\top \right)^{-1} \sum_{t=1}^T \vtheta_t (y_t - \Bar{o} - \Bar{\vb}^\top \vtheta^*)$ from Appendix \ref{sec:2SLS-derivations}. Plugging this into $\|\widehat{\vtheta} - \vtheta^*\|_2$, we get
\begin{equation*}
\begin{aligned}
    &\left\| \widehat{\vtheta} - \vtheta^* \right\|_2 =\\ &\left\| \left(\sum_{t=1}^T \vtheta_t (\vx_t - \Bar{\vb})^\top \right)^{-1} \left(\sum_{t=1}^T \vtheta_t (y_t - \Bar{o} - \Bar{\vb}^\top \vtheta^*) \right) - \vtheta^* \right\|_2
\end{aligned}
\end{equation*}
Next, we substitute in our expression for $y_t$ and simplify, obtaining
\begin{equation*} \label{eq:bound-main}
\begin{aligned}
   &\| \widehat{\vtheta} - \vtheta^* \|_2\\
   &= \left\| \left(\sum_{t=1}^T \vtheta_t (\vx_t - \Bar{\vb})^\top \right)^{-1} \right.\\ &\left. \left(\sum_{t=1}^T \vtheta_t (\vx_t^\top \vtheta^* + o_t - \Bar{o} - \Bar{\vb}^\top \vtheta^*) \right) - \vtheta^* \right\|_2\\
   &= \left\| \left(\sum_{t=1}^T \vtheta_t (\vx_t - \Bar{\vb})^\top \right)^{-1} \left(\sum_{t=1}^T \vtheta_t (\vx_t - \Bar{\vb})^\top \vtheta^* \right) \right.\\ &\left.+ \left(\sum_{t=1}^T \vtheta_t (\vx_t- \Bar{\vb})^\top \right)^{-1} \left(\sum_{t=1}^T \vtheta_t (o_t - \Bar{o}) \right) - \vtheta^*\right\|_2\\
   &= \left\| \vtheta^* + \left(\sum_{t=1}^T \vtheta_t (\vx_t - \Bar{\vb})^\top \right)^{-1} \left(\sum_{t=1}^T \vtheta_t (o_t - \Bar{o}) \right) - \vtheta^*\right\|_2\\
   &= \left\|\left(\sum_{t=1}^T \vtheta_t (\vx_t - \Bar{\vb})^\top \right)^{-1} \left(\sum_{t=1}^T \vtheta_t (o_t - \Bar{o}) \right) \right\|_2\\
   & \leq \left\|\left(\sum_{t=1}^T \vtheta_t (\vx_t - \Bar{\vb})^\top \right)^{-1} \right\|_2 \left \|\sum_{t=1}^T \vtheta_t (o_t - \Bar{o})  \right\|_2\\
   & \leq \frac{\left \|\sum_{t=1}^T \vtheta_t (o_t - \Bar{o})  \right\|_2}{\sigma_{min}\left(\sum_{t=1}^T \vtheta_t (\vx_t - \Bar{\vb})^\top \right)}
\end{aligned}
\end{equation*}
We now bound the numerator and denominator separately with high probability.
\subsection{Bound on numerator}
\begin{equation*}
\begin{aligned}
    &\left \|\sum_{t=1}^T \vtheta_t (o_t - \Bar{o})  \right\|_2 = \left \|\sum_{t=1}^T \vtheta_t (o_t - \mathbb{E}[o_t] + \mathbb{E}[o_t] - \Bar{o})  \right\|_2\\
    &\leq \left \|\sum_{t=1}^T \vtheta_t (o_t - \mathbb{E}[o_t]) \right \|_2 + \left \|\sum_{t=1}^T \vtheta_t (\mathbb{E}[o_t] - \Bar{o})  \right\|_2\\
\end{aligned}
\end{equation*}
\subsubsection{Bound on first term}
\begin{equation*}
\begin{aligned}
    &\left \|\sum_{t=1}^T \vtheta_t (o_t - \mathbb{E}[o_t]) \right \|_2\\ &= \left(\sum_{j=1}^m \left(\sum_{t=1}^T \theta_{t,j} (o_t - \mathbb{E}[o_t]) \right)^2\right)^{1/2}
\end{aligned}
\end{equation*}

Since $(o_t - \mathbb{E}[o_t])$ is a zero-mean bounded random variable with variance parameter $\sigma_g^2$, the product $\theta_{t,j} (o_t - \mathbb{E}[o_t])$ will also be a zero-mean bounded random variable with variance at most $\beta^2 \sigma_g^2$. In order to bound $\left(\sum_{j=1}^m \left(\sum_{t=1}^T \theta_{t,j} (o_t - \mathbb{E}[o_t]) \right)^2\right)^{1/2}$ with high probability, we make use of the following lemma. Note that bounded random variables are sub-Gaussian random variables.

\begin{lemma}[High probability bound on the sum of unbounded sub-Gaussian random variables] \label{lemma:chernoff}
Let $x_t \sim \text{subG}(0, \sigma^2)$. For any $\delta \in (0, 1)$, with probability at least $1 - \delta$,
\begin{equation*}
    \Big|\sum_{t=1}^T x_t\Big| \leq \sigma \sqrt{2T\log(1/\delta)}
\end{equation*}
\end{lemma}

Applying Lemma \ref{lemma:chernoff} to $\left(\sum_{j=1}^m \left(\sum_{t=1}^T \theta_{t,j} (o_t - \mathbb{E}[o_t]) \right)^2\right)^{1/2}$, we get
\begin{align*}
    &\sqrt{\sum_{j=1}^m \left(\sum_{t=1}^T \theta_{t,j} (o_t - \mathbb{E}[o_t]) \right)^2}\\ 
    &\leq \sqrt{\sum_{j=1}^m \left(\beta \sigma_g \sqrt{2T\log(1/\delta_j)} \right)^2}\\
    &\leq \sqrt{\sum_{j=1}^m \beta^2 \sigma_g^2 2T\log(m/\delta)} \tag{by a union bound, where $\delta_j=\delta/m$ for all $j$}\\
    &\leq \beta \sigma_g \sqrt{2Tm\log(m/\delta)}\\
\end{align*}
with probability at least $1 - \delta$.

\subsubsection{Bound on second term}
\begin{equation*}
\begin{aligned}
    &\left \|\sum_{t=1}^T \vtheta_t (\mathbb{E}[o_t] - \Bar{o})  \right\|_2\\ 
    &= \left \|\sum_{t=1}^T \vtheta_t \left(\mathbb{E}[o_t] - \frac{1}{T} \sum_{s=1}^T o_s \right)  \right\|_2\\
    &= \left \|\sum_{t=1}^T \vtheta_t \frac{1}{T} \sum_{s=1}^T \left(\mathbb{E}[o_t] - o_s \right)  \right\|_2\\
    &= \left(\sum_{j=1}^m \left( \sum_{t=1}^T \theta_{t,j} \frac{1}{T} \sum_{s=1}^T \left(\mathbb{E}[o_t] - o_s \right)\right)^2\right)^{1/2}\\
    &\leq \left(\sum_{j=1}^m \left( \sum_{t=1}^T |\theta_{t,j}| \frac{1}{T} \left| \sum_{s=1}^T \mathbb{E}[o_t] - o_s \right| \right)^2\right)^{1/2}\\
\end{aligned}
\end{equation*}

After applying Lemma \ref{lemma:chernoff}, we get
\begin{equation*}
\begin{aligned}
    &\left \|\sum_{t=1}^T \vtheta_t (\mathbb{E}[o_t] - \Bar{o})  \right\|_2\\ 
    &\leq \left(\sum_{j=1}^m \left( \sum_{t=1}^T |\theta_{t,j}| \frac{1}{T} \sigma_g \sqrt{2T \log(1/\delta_j)} \right)^2\right)^{1/2}\\
    &\leq \left(\sum_{j=1}^m \left( \beta \sigma_g \sqrt{2T \log(1/\delta_j)} \right)^2\right)^{1/2}\\
    &\leq \left(\sum_{j=1}^m \beta^2 \sigma_g^2 2T \log(m/\delta) \right)^{1/2}\\
    &\leq \beta \sigma_g \sqrt{2Tm \log(m/\delta)}\\
\end{aligned}
\end{equation*}
with probability at least $1 - \delta$
\subsection{Proof of Corollary \ref{cor:theta-bound}} \label{sec:denom}
Next let's bound the denominator. By plugging in the expression for $\vx_t$, we see that 
\begin{equation*}
\begin{aligned}
   &\sigma_{min}\left(\sum_{t=1}^T \vtheta_t (\vx_t - \Bar{\vb})^\top \right)\\ 
   &= \sigma_{min}\left(\sum_{t=1}^T \vtheta_t (\vb_t- \Bar{\vb})^\top + \vtheta_t \vtheta_t^\top \mathcal{E}_t \mathcal{E}_t^\top \right)\\
   &= \sigma_{min}\left(A + B\right),
\end{aligned}
\end{equation*}
where $A = \sum_{t=1}^T \vtheta_t (\vb_t- \Bar{\vb})^\top$ and $B = \sum_{t=1}^T \vtheta_t \vtheta_t^\top \mathcal{E}_t \mathcal{E}_t^\top$. By definition, 
\begin{equation*}
    \sigma_{min}(A+B) = \min_{\va, \|\va\|_2 = 1} \|(A + B)\va\|_2.
\end{equation*}

Via the triangle inequality, 
\begin{equation*} \label{eq:triangle}
\begin{aligned}
    \sigma_{min}(A+B) 
    &\geq \min_{\va, \|\va\|_2 = 1} \left( \|B\va\|_2 - \|A\va\|_2 \right)\\
    &\geq \min_{\va, \|\va\|_2 = 1} \|B\va\|_2 - \|A\|_2\\
    &\geq \sigma_{min}(B) - \|A\|_2\\.
\end{aligned}
\end{equation*}

\subsubsection{Bounding $\|A\|_2$}
\begin{equation*}
\begin{aligned}
    \|A\|_2 &= \left \| \sum_{t=1}^T \vtheta_t (\vb_t - \mathbb{E}[\vb_t] + \mathbb{E}[\vb_t] - \Bar{\vb})^\top \right \|_2\\
    &\leq \left \| \sum_{t=1}^T \vtheta_t (\vb_t - \mathbb{E}[\vb_t])^\top \right \|_2 + \left \| \sum_{t=1}^T \vtheta_t (\mathbb{E}[\vb_t] - \Bar{\vb})^\top \right \|_2\\
\end{aligned}
\end{equation*}

\xhdr{Bound on first term}
\begin{equation*}
\begin{aligned}
    &\left \| \sum_{t=1}^T \vtheta_t (\vb_t - \mathbb{E}[\vb_t])^\top \right \|_2 
    \leq \left \| \sum_{t=1}^T \vtheta_t (\vb_t - \mathbb{E}[\vb_t])^\top \right \|_F\\
    &\leq \left( \sum_{i=1}^m \sum_{j=1}^m \left ( \sum_{t=1}^T \theta_{t,i} (z_{t,j} - \mathbb{E}[z_{t,j}]) \right )^2 \right)^{1/2}\\
\end{aligned}
\end{equation*}
Notice that $\theta_{t,i} (z_{t,j} - \mathbb{E}[z_{t,j}])$ is a zero-mean bounded random variable with variance at most $\beta^2 \sigma_z^2$. Applying Lemma \ref{lemma:chernoff}, we can see that
\begin{equation*}
\begin{aligned}
    &\left \| \sum_{t=1}^T \vtheta_t (\vb_t - \mathbb{E}[\vb_t])^\top \right \|_2\\ 
    &\leq \left( \sum_{i=1}^m \sum_{j=1}^m \left ( \beta \sigma_z \sqrt{2T \log(1/\delta_{i,j})} \right )^2 \right)^{1/2}\\
    &\leq \left( \sum_{i=1}^m \sum_{j=1}^m \beta^2 \sigma_z^2 2T \log(m^2/\delta) \right)^{1/2}\\
    &\leq \left( m^2 \beta^2 \sigma_z^2 2T \log(m^2/\delta) \right)^{1/2}\\
    &\leq m \beta \sigma_z \sqrt{2T \log(m^2/\delta)}\\
\end{aligned}
\end{equation*}
with probability at least $1 - \delta$.

\xhdr{Bound on second term}
\begin{equation*}
\begin{aligned}
    &\left \| \sum_{t=1}^T \vtheta_t (\mathbb{E}[\vb_t] - \Bar{\vb})^\top \right \|_2\\ 
    &= \left \| \sum_{t=1}^T \vtheta_t \frac{1}{T} \sum_{s=1}^T (\mathbb{E}[\vb_t] - \vb_j)^\top \right \|_2\\
    &\leq \left \| \sum_{t=1}^T \vtheta_t \frac{1}{T} \sum_{s=1}^T (\mathbb{E}[\vb_t] - \vb_j)^\top \right \|_F\\
    &\leq \left( \sum_{i=1}^m \sum_{j=1}^m \left ( \sum_{t=1}^T \theta_{t,i} \frac{1}{T} \sum_{s=1}^T (\mathbb{E}[z_{t,j}] - z_j) \right )^2 \right)^{1/2}\\
    &\leq \left( \sum_{i=1}^m \sum_{j=1}^m \left ( \sum_{t=1}^T |\theta_{t,i}| \frac{1}{T} \left|\sum_{s=1}^T (\mathbb{E}[z_{t,j}] - z_j) \right| \right )^2 \right)^{1/2}\\
\end{aligned}
\end{equation*}

By applying Lemma \ref{lemma:chernoff}, we obtain
\begin{equation*}
\begin{aligned}
    &\left \| \sum_{t=1}^T \vtheta_t (\mathbb{E}[\vb_t] - \Bar{\vb})^\top \right \|_2\\ 
    &\leq \left( \sum_{i=1}^m \sum_{j=1}^m \left ( \sum_{t=1}^T |\theta_{t,i}| \frac{1}{T} \sigma_{z} \sqrt{2T \log(1/\delta_{i,j})} \right )^2 \right)^{1/2}\\
    &\leq \left( \sum_{i=1}^m \sum_{j=1}^m \left ( \beta \sigma_{z} \sqrt{2T \log(1/\delta_{i,j})} \right )^2 \right)^{1/2}\\
    &\leq \left( \sum_{i=1}^m \sum_{j=1}^m \beta^2 \sigma_{z}^2 2T \log(m^2/\delta) \right)^{1/2}\\
    &\leq m \beta \sigma_{z} \sqrt{2T \log(m^2/\delta)}\\
\end{aligned}
\end{equation*}

\subsubsection{Bounding $\sigma_{min}(B)$}

Next we bound $\sigma_{min}(B) = \sigma_{min}(\sum_{t=1}^T \vtheta_t \vtheta_t^\top \mathcal{E}_t \mathcal{E}_t^\top)$. We can write $\mathcal{E}_t \mathcal{E}_t^\top$ as $\mathbb{E}[\mathcal{E}_t \mathcal{E}_t^\top] + \epsilon_t$. Note that since each element of $\mathcal{E}_t$ is bounded, each element of $\epsilon_t \in \mathbb{R}^{m \times m}$ will be bounded as well. Using this formulation,
\begin{equation*}
\begin{aligned}
    \sigma_{min}(B)
    &= \sigma_{min} \left(\sum_{t=1}^T \vtheta_t \vtheta_t^\top( \mathbb{E}[\mathcal{E}_t \mathcal{E}_t^\top] + \epsilon_t) \right)\\
    &= \sigma_{min} \left(\sum_{t=1}^T \vtheta_t \vtheta_t^\top \mathbb{E}[\mathcal{E}_t \mathcal{E}_t^\top] + \sum_{t=1}^T \vtheta_t \vtheta_t^\top \epsilon_t) \right)\\
    &\geq \sigma_{min} \left(\sum_{t=1}^T \vtheta_t \vtheta_t^\top \mathbb{E}[\mathcal{E}_t \mathcal{E}_t^\top] \right) - \left\|\sum_{t=1}^T \vtheta_t \vtheta_t^\top \epsilon_t \right\|_2\\
    &\geq \sigma_{min} \left(\sum_{t=1}^T \vtheta_t \vtheta_t^\top \mathbb{E}[\mathcal{E}_t \mathcal{E}_t^\top] \right) - \left\|\sum_{t=1}^T \vtheta_t \vtheta_t^\top \epsilon_t \right\|_F\\
\end{aligned}
\end{equation*}
We proceed by bounding each term separately.

\xhdr{Bound on first term}
\begin{equation*}
\begin{aligned}
    &\sigma_{min} \left(\sum_{t=1}^T \vtheta_t \vtheta_t^\top \mathbb{E}[\mathcal{E}_t \mathcal{E}_t^\top] \right)\\ &\geq \sigma_{min}(\mathbb{E}[\mathcal{E}_t \mathcal{E}_t^\top])\sigma_{min}(\sum_{t=1}^T \vtheta_t \vtheta_t^\top)
\end{aligned}
\end{equation*}
Let $c = \sigma_{min}(\mathbb{E}[\mathcal{E}_t \mathcal{E}_t^\top])$. We assume that $\mathcal{E}_t$ is distributed such that $c > 0$. Therefore, 
\begin{equation*}
    \sigma_{min} \left(\sum_{t=1}^T \vtheta_t \vtheta_t^\top \mathbb{E}[\mathcal{E}_t \mathcal{E}_t^\top] \right) \geq c\sigma_{min}(\sum_{t=1}^T \vtheta_t \vtheta_t^\top).
\end{equation*}

Next, we use the matrix Chernoff bound to bound $c \sigma_{min}(\sum_{t=1}^T \vtheta_t \vtheta_t^\top) = c \lambda_{min}(\sum_{t=1}^T \vtheta_t \vtheta_t^\top)$ with high probability.


\begin{theorem}[Matrix Chernoff]
    Consider a finite sequence $\{X_t\}_{t=1}^T$ of independent, random, Hermitian matrices with common dimension $d$. Assume that 
    \begin{equation*}
        0 \leq \lambda_{min}(X_t) \text{ and } \lambda_{max}(X_t) \leq L \text{ for each index } t 
    \end{equation*}
    Introduce the random matrix 
    \begin{equation*}
        Y = \sum_{t=1}^T X_t.
    \end{equation*}
    Define the minimum eigenvalue $\mu_{min}$ of the expectation $\mathbb{E}[Y]$:
    \begin{equation*}
    \begin{aligned}
        \mu_{min} &= \lambda_{min}(\mathbb{E}[Y]) = \lambda_{min} \left(\sum_{t=1}^T \mathbb{E}[X_t] \right)
    \end{aligned}
    \end{equation*}
    Then, 
    \begin{equation*}
        P(\lambda_{min}(Y) \leq (1 - \epsilon)\mu_{min}) \leq d\left(\frac{e^{-\epsilon}}{(1 - \epsilon)^{1-\epsilon}} \right)^{\mu_{min}/L}
    \end{equation*}
    for $\epsilon \in [0, 1)$.
\end{theorem}

Let $Y = \sum_{t=1}^T \vtheta_t \vtheta_t^\top$. In our setting, 
\begin{equation*}
\begin{aligned}
    \mu_{min} &= \lambda_{min}\left(\sum_{t=1}^T \mathbb{E}[\vtheta_t \vtheta_t^\top] \right)\\
    &= T \lambda_{min}\left(\mathbb{E}[\vtheta_t \vtheta_t^\top] \right)\\
    &= T \lambda_{min}\left(\sigma_{\theta}^2 \mathbb{I}_{m \times m} + \mathbb{E}[\vtheta_t] \mathbb{E}[\vtheta_t^\top] \right)\\
\end{aligned}
\end{equation*}

$\sigma_{\theta}^2 \mathbb{I}_{m \times m}$ and $\mathbb{E}[\vtheta_t] \mathbb{E}[\vtheta_t^\top]$ commute, so

\begin{equation*}
\begin{aligned}
    \mu_{min} &= T \left( \lambda_{min}\left(\sigma_{\theta}^2 \mathbb{I}_{m \times m} \right) + \lambda_{min}\left(\mathbb{E}[\vtheta_t] \mathbb{E}[\vtheta_t^\top] \right) \right)\\
    &= T \lambda_{min}\left(\sigma_{\theta}^2 \mathbb{I}_{m \times m} \right)\\
    &= T \sigma_{\theta}^2 \lambda_{min}\left(\mathbb{I}_{m \times m} \right)\\
    &= T \sigma_{\theta}^2\\
\end{aligned}
\end{equation*}

\begin{equation*}
    \lambda_{max}(\vtheta_t \vtheta_t^\top) = \beta m,
\end{equation*}
so let $L = \beta m$.

Picking $\epsilon = 1/2$ and applying the matrix Chernoff bound to $ \lambda_{min}(\sum_{t=1}^T \vtheta_t \vtheta_t^\top)$, we obtain
\begin{equation*}
    P\left(\lambda_{min} \left(\sum_{t=1}^T \vtheta_t \vtheta_t^\top \right)\leq \frac{1}{2}T \sigma_{\theta}^2 \right) \leq d\left(\frac{1}{2}e \right)^{-\frac{T \sigma_{\theta}^2}{2\beta m}}
\end{equation*}
By rearranging terms, we see that if $T \geq \frac{2 \beta m}{\sigma_{\theta}^2 \log{ \frac{1}{2}e}}\log{\frac{d}{\delta}}$, then
\begin{equation*}
    \lambda_{min} \left(\sum_{t=1}^T \vtheta_t \vtheta_t^\top \right) \geq \frac{1}{2}T \sigma_{\theta}^2
\end{equation*}
with probability at least $1 - \delta$.

\xhdr{Bound on second term}
\begin{equation*}
\begin{aligned}
    \left\|\sum_{t=1}^T \vtheta_t \vtheta_t^\top \epsilon_t \right\|_F
    &= \left(\sum_{i=1}^m \sum_{j=1}^m \left(\sum_{t=1}^T \theta_{t,i} \theta_{t,j} \epsilon_{t,i,j} \right)^2 \right)^{1/2}\\    
\end{aligned}
\end{equation*}
Since each $\epsilon_{t,i,j}$ is a bounded zero-mean random variable, $\theta_{t,i} \theta_{t,j} \epsilon_{t,i,j}$ is also a bounded zero-mean random variable, with variance at most $\beta^4 \sigma_{\mathcal{E}}^2$ We can now apply Lemma \ref{lemma:chernoff}:
\begin{equation*}
\begin{aligned}
    &\left\|\sum_{t=1}^T \vtheta_t \vtheta_t^\top \epsilon_t \right\|_F\\
    &\leq \left(\sum_{i=1}^m \sum_{j=1}^m \left(\beta^2 \sigma_{\mathcal{E}} \sqrt{2T \log (1/\delta_{i,j})} \right)^2 \right)^{1/2}\\    
    &\leq \left(\sum_{i=1}^m \sum_{j=1}^m \beta^4 \sigma_{\mathcal{E}}^2 2T \log (m^2/\delta) \right)^{1/2}\\ 
    &\leq \left(m^2 \beta^4 \sigma_{\mathcal{E}}^2 2T \log (m^2/\delta) \right)^{1/2}\\ 
    &\leq m \beta^2 \sigma_{\mathcal{E}} \sqrt{2T \log (m^2/\delta)}\\ 
\end{aligned}
\end{equation*}
with probability at least $1 - \delta$.

\xhdr{Putting everything together}

Putting everything together, we have that 
\begin{equation*}
\begin{aligned}
&\| \widehat{\vtheta} - \vtheta^* \|_2 \leq\\ &\frac{2\beta \sigma_g \sqrt{2m\log(m/\delta)} }{\frac{1}{2}c \sqrt{T} \sigma_{\theta}^2 - m \beta^2 \sigma_{\mathcal{E}} \sqrt{2 \log (m^2/\delta)} - 2m \beta \sigma_z \sqrt{2 \log(m^2/\delta)}}
\end{aligned}
\end{equation*}
with probability at least $1 - 6\delta$.
\section{Individual Fairness Derivations}\label{sec:fairness-derivations}

\subsection{Proof of Theorem \ref{thm:IF}}
\begin{proof}
\begin{equation*}
\begin{aligned}
    &|\hat{y} - \hat{y}'| =\\ &|(\vb - \vb')^\top \vtheta^* + \vtheta^{*\top} (\mathcal{E}\mathcal{E}^\top - \mathcal{E}' \mathcal{E}^{'\top}) \vtheta^* |\\
    &= |(\vb_{\mathcal{C}} - \vb_{\mathcal{C}}')^\top \vtheta^* + \vtheta^{*\top} ((\mathcal{E}\mathcal{E}^\top)_{\mathcal{C}} - (\mathcal{E}'\mathcal{E}^{'\top})_{\mathcal{C}} ) \vtheta^* |\\
    &\leq \| \vb_{\mathcal{C}} - \vb_{\mathcal{C}}' \|_2 \| \vtheta^* \|_2 + \| \vtheta^* \|_2 \| ((\mathcal{E}\mathcal{E}^\top)_{\mathcal{C}} - (\mathcal{E}'\mathcal{E}^{'\top})_{\mathcal{C}}) \vtheta^* \|_2\\
    &\leq \| \vb_{\mathcal{C}} - \vb_{\mathcal{C}}' \|_2 + \max_{\vtheta, \|\vtheta\|_2 = 1} \| ((\mathcal{E}\mathcal{E}^\top)_{\mathcal{C}} - (\mathcal{E}'\mathcal{E}^{'\top})_{\mathcal{C}}) \vtheta^* \|_2\\
    &\leq \| \vb_{\mathcal{C}} - \vb_{\mathcal{C}}' \|_2 + \| ((\mathcal{E}\mathcal{E}^\top)_{\mathcal{C}} - (\mathcal{E}'\mathcal{E}^{'\top})_{\mathcal{C}})\|_2\\
\end{aligned}
\end{equation*}
\end{proof}

\subsection{Proof of Theorem \ref{thm:IF-bound}}
\begin{proof}
    Let 
    \begin{align*}
        b_{\widetilde{\mathcal{C}},i} &=  
        \begin{cases}
            b_i & \text{if } i \not\in \mathcal{C}\\
            0 & \text{otherwise}
        \end{cases}
        , &\\
        (\mathcal{E}\mathcal{E}^\top)_{\widetilde{\mathcal{C}},ij} &=  
        \begin{cases}
            (\mathcal{E}\mathcal{E}^\top)_{ij} & \text{if } i,j \not\in \mathcal{C}\\
            0 & \text{otherwise}.
        \end{cases}    
    \end{align*}

    \begin{equation*}
    \begin{aligned}
        &|\hat{y} - \hat{y}'| = |(\vb_{\mathcal{C}} - \vb_{\mathcal{C}}')^\top \vtheta + \vtheta^\top ((\mathcal{E}\mathcal{E}^\top)_\mathcal{C} - (\mathcal{E}'\mathcal{E}^{'\top})_\mathcal{C})\vtheta\\
        &+ (\vb_{\widetilde{\mathcal{C}}} - \vb_{\widetilde{\mathcal{C}}}')^\top \vtheta + \vtheta^\top ((\mathcal{E}\mathcal{E}^\top)_{\widetilde{\mathcal{C}}} - (\mathcal{E}'\mathcal{E}^{'\top})_{\widetilde{\mathcal{C}}})\vtheta|\\
        &= |(\vb_{\widetilde{\mathcal{C}}} - \vb_{\widetilde{\mathcal{C}}}')^\top \vtheta + \vtheta^\top ((\mathcal{E}\mathcal{E}^\top)_{\widetilde{\mathcal{C}}} - (\mathcal{E}'\mathcal{E}^{'\top})_{\widetilde{\mathcal{C}}})\vtheta|\\
        &= \left| \sum_{i \not \in \mathcal{C}} (b_i - b_i') \theta_i + \sum_{i \not \in \mathcal{C}} \sum_{j \not \in \mathcal{C}} ((\mathcal{E}\mathcal{E}^\top)_{ij} - (\mathcal{E}'\mathcal{E}^{'\top})_{ij}) \theta_i \theta_j \right|\\
    \end{aligned}
    \end{equation*}
\end{proof}

\subsection{Example \ref{ex:fairness} Derivations}
\begin{equation*}
\begin{aligned}
    d(\vu, \vu') &= \|b_{\mathcal{C}} - b'_{\mathcal{C}}\|_2 + \| (\mathcal{E} \mathcal{E}^\top)_\mathcal{C} - (\mathcal{E}' \mathcal{E}'^\top)_\mathcal{C} \|_2\\
    &= \| (\boldsymbol{\delta} - \boldsymbol{\delta}')_\mathcal{C} I_{n \times n} \|_2 = \|\mathbf{0}_{n \times n}\|_2 = 0,\\
\end{aligned}
\end{equation*}
where 
\begin{equation*}
    \delta_{\mathcal{C}, i} = 
    \begin{cases}
    \delta_i & \text{if } i \in \mathcal{C}\\
    0 & \text{otherwise}.
    \end{cases}
\end{equation*}

\begin{equation*}
\begin{aligned}
    |\hat{y} - \hat{y}'| &= \left| 0 + \sum_{i \not \in \mathcal{C}} (n - 0) \theta_i^2 \right|
    = n \sum_{i =1}^{n/2} \theta_i^2\\
\end{aligned}
\end{equation*}

\section{Agent outcome maximization derivations}

\subsection{Derivation of $\vtheta^{AO}$} \label{sec:AO-derivation}
\begin{equation*}
\begin{aligned}
    \vtheta^{AO} = \arg \max_{\vtheta \in \mathcal{S}} \quad & \mathbb{E}[y_t]\\
\end{aligned}
\end{equation*}

Substituting in for $y_t$:
\begin{equation*}
\begin{aligned}
    \vtheta^{AO} = \arg \max_{\vtheta \in \mathcal{S}} \quad & \mathbb{E}[\vx_t^\top \vtheta^* + o_t]\\
\end{aligned}
\end{equation*}

\begin{equation*}
\begin{aligned}
    \vtheta^{AO} = \arg \max_{\vtheta \in \mathcal{S}} \quad & \mathbb{E}[\vx_t^\top \vtheta^*] + \mathbb{E}[o_t]\\
\end{aligned}
\end{equation*}

\begin{equation*}
\begin{aligned}
    \vtheta^{AO} = \arg \max_{\vtheta \in \mathcal{S}} \quad & \mathbb{E}[\vx_t^\top \vtheta^*]\\
\end{aligned}
\end{equation*}

Substitute in for $\vx_t$:
\begin{equation*}
\begin{aligned}
    \vtheta^{AO} = \arg \max_{\vtheta \in \mathcal{S}} \quad & \mathbb{E}[(\vb_t^\top + \vtheta^\top \mathcal{E}_t \mathcal{E}_t^\top) \vtheta^*]\\
\end{aligned}
\end{equation*}

\begin{equation*}
\begin{aligned}
    \vtheta^{AO} = \arg \max_{\vtheta \in \mathcal{S}} \quad & \mathbb{E}[\vb_t^\top \vtheta^*] + \mathbb{E}[\vtheta^\top \mathcal{E}_t \mathcal{E}_t^\top \vtheta^*]\\
\end{aligned}
\end{equation*}

\begin{equation*}
\begin{aligned}
    \vtheta^{AO} = \arg \max_{\vtheta \in \mathcal{S}} \quad & \mathbb{E}[\vtheta^\top \mathcal{E}_t \mathcal{E}_t^\top \vtheta^*]\\
\end{aligned}
\end{equation*}

\begin{equation*}
\begin{aligned}
    \vtheta^{AO} &= \arg \max_{\vtheta \in \mathcal{S}} \quad \vtheta^\top \mathbb{E}[\mathcal{E}_t \mathcal{E}_t^\top]\vtheta^*\\
    &= \arg \max_{\vtheta \in \mathcal{S}} \quad \sum_{i \in \mathcal{C}} \sum_{j \in \mathcal{C}} \sum_{k=1}^d \E[w_{ik} w_{jk}] \theta_i \theta_j^*\\ &+ \sum_{i \not\in \mathcal{C}} \sum_{j \in \mathcal{C}} \sum_{k=1}^d \E[w_{ik} w_{jk}] \theta_i \theta_j^*
\end{aligned}
\end{equation*}

\section{Predictive risk minimization derivations}
\subsection{Population gradient derivation} \label{sec:pop-grad-derivation}
The gradient of the population risk function $f(\vtheta_t) = \mathbb{E}[(\widehat{y}_t - y_t)^2]$ can be derived as follows

\begin{equation*}
\begin{aligned}
    \nabla_{\vtheta_t}f(\vtheta_t) &= \mathbb{E}[\nabla_{\vtheta_t} (\widehat{y}_t - y_t)^2]\\
    &= 2\mathbb{E}[(\widehat{y}_t - y_t)\nabla_{\vtheta_t} (\widehat{y}_t - y_t)]\\
    &= 2\mathbb{E}[(\widehat{y}_t - y_t)\nabla_{\vtheta_t} (\vx_t^\top \vtheta_t - \vx_t^\top \vtheta^* - o_t)]\\
    &= 2\mathbb{E}[(\widehat{y}_t - y_t)\nabla_{\vtheta_t} (\vx_t^\top (\vtheta_t - \vtheta^*))]\\
    &= 2\mathbb{E}[(\widehat{y}_t - y_t)\nabla_{\vtheta_t} ((\vb_t^\top + \vtheta_t \mathcal{E}_t \mathcal{E}_t^\top) (\vtheta_t - \vtheta^*))]\\
    &= 2\mathbb{E}[(\widehat{y}_t - y_t)(\vb_t +  \mathcal{E}_t \mathcal{E}_t^\top (2\vtheta_t - \vtheta^*))]\\
    &= 2\mathbb{E}[(\widehat{y}_t - y_t)(\vx_t +  \mathcal{E}_t \mathcal{E}_t^\top (\vtheta_t - \vtheta^*))]\\
\end{aligned}
\end{equation*}

\section{Omitted experiments}
\label{sec:omitted-experiments}
In this section, we present additional details for our experiments in \Cref{sec:experiment}. At the end, we provide more information regarding the dataset and computation resources used.

\begin{figure}[ht]
    \centering
    \includegraphics[width=\linewidth]{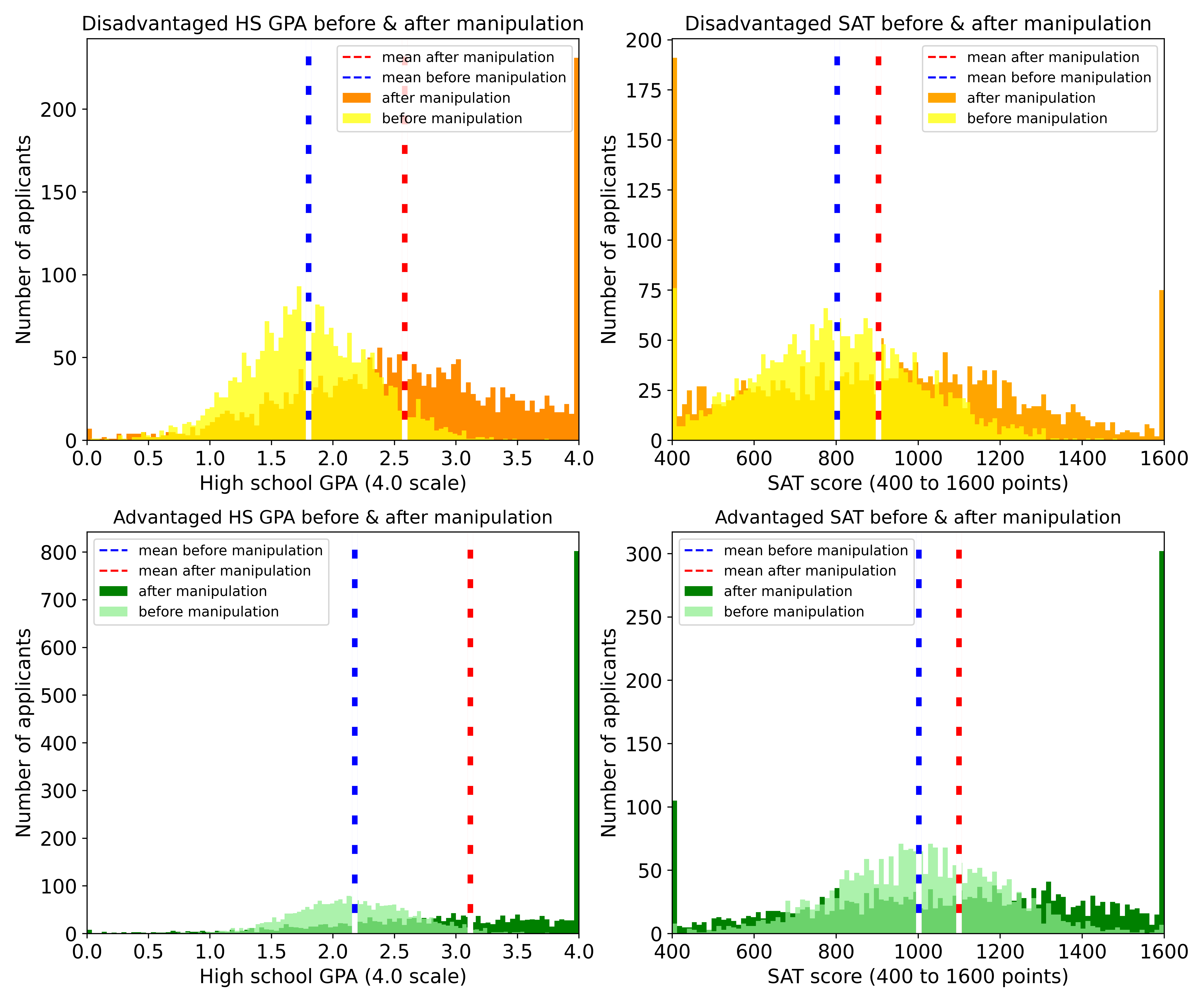}
    \caption{Distributions of unobserved features $\vb$ (in lighter colors), i.e. initial HS GPA (two left figures) and SAT (two right figures), and observed features $\vx$ (darker colors) for disadvantaged (two top figures in yellow and orange) and advantaged students (two bottom figures in green).}
    \label{fig:features-shift}
\end{figure}

\begin{figure}
    \centering
    \includegraphics[width=0.6\linewidth]{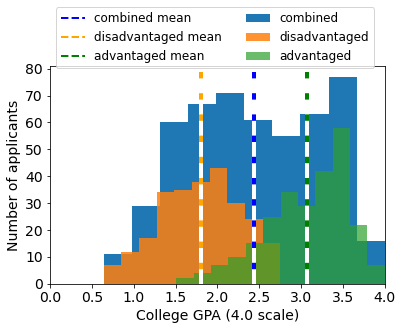}
    \caption{Distribution of college GPAs (outcomes $y$) for disadvantaged students (orange), advantaged students (green), and both combined (blue).}
    \label{fig:outcomes}
\end{figure}

\begin{figure}
    \centering
    \includegraphics[width=0.5\linewidth]{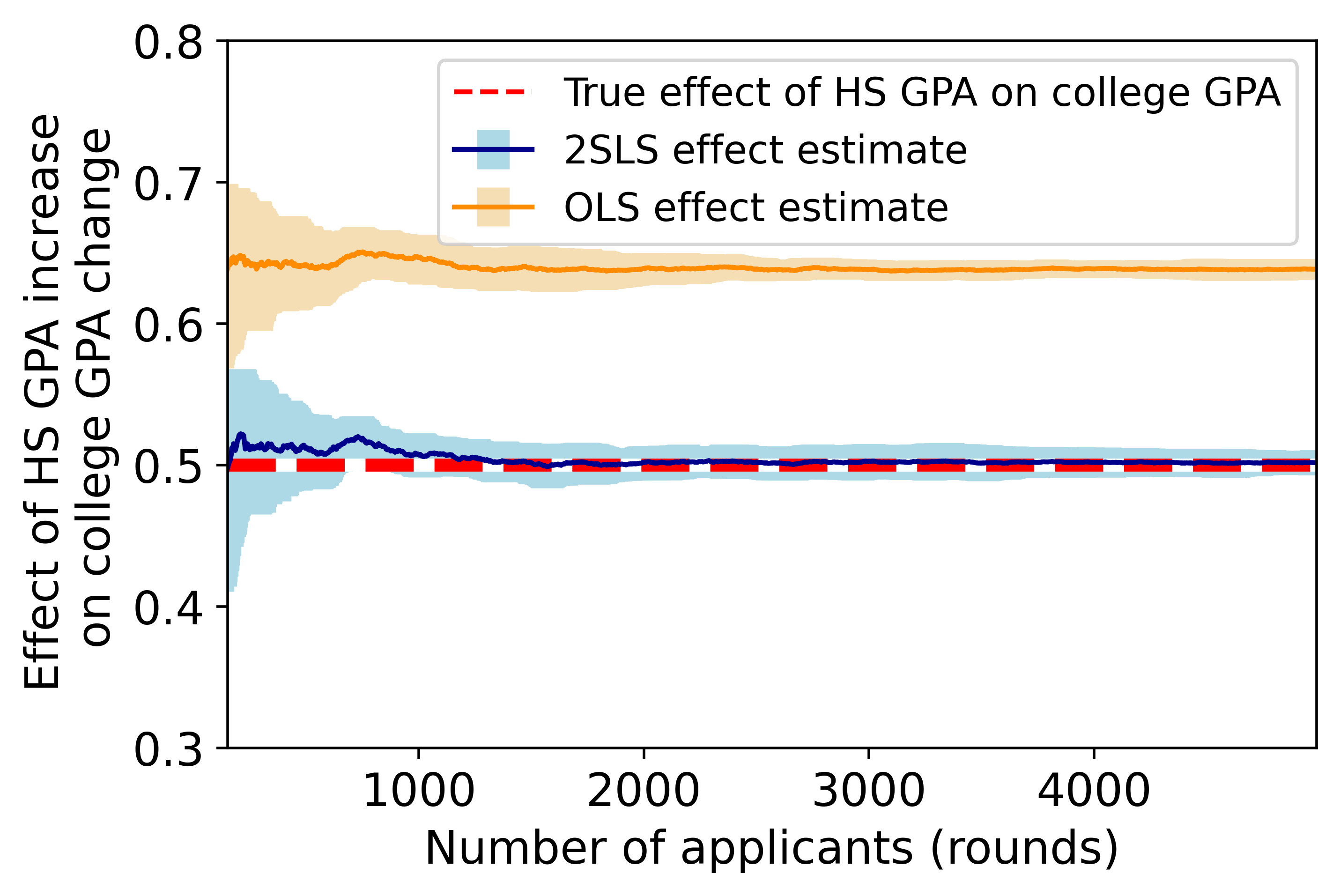}
    \caption{OLS versus 2SLS estimates for high school GPA effect on college GPA over 5000 rounds. Results are averaged over 10 runs, with the error bars (in lighter colors) representing one standard deviation. The red dashed line is the true causal effect of each high school GPA on college GPA.}
    \label{fig:gpa-estimates}
\end{figure}

\subsection{University admissions full experimental description} We construct a semi-synthetic dataset based on an example of university admissions with disadvantaged and advantaged students from \citet{hu2018disparate}. From a real dataset of the high school (HS) GPA, SAT score, and college GPA of 1000 college students, we estimate the causal effect of observed features $[\text{SAT}, \text{HS GPA}]$ on college GPA to be $\vtheta^*=[0.00085, 0.49262]^\top$ using OLS (which is assumed to be consistent, since we have yet to modify the data to include confounding). We then use this dataset to construct synthetic data which looks similar, yet incorporates confounding factors. For simplicity, we let the true causal effect parameters $\vtheta^*=[0, 0.5]^\top$. That is, we assume there is a significant causal relationship between college performance and HS GPA, but not SAT score.\footnote{Though this assumption may be contentious, it is based on existing research \cite{allensworth2020gpasat}.} We consider two types of student backgrounds, those from a \textit{disadvantaged group} and those from an \textit{advantaged group}. We assume disadvantaged applicants have, on average, lower HS GPA and SAT $\vb_t$, lower baseline college GPA $o_t$, and require more effort to improve observable features (reflected in $\mathcal{E}_t$): this could be due to disadvantaged groups being systemically underserved, marginalized, or abjectly discriminated against (and the converse for advantaged groups). Initial features $\vb_t$ are constructed as such: For any disadvantaged applicant $t$, their initial SAT features $z_t^{\text{SAT}} \sim \mathcal{N}(800,200)$ and initial HS GPA $z_t^{\text{HS GPA}} \sim \mathcal{N}(1.8,0.5)$. For any advantaged applicant $t$, $z_t^{\text{SAT}} \sim \mathcal{N}(1000,200)$ and $z_t^{\text{HS GPA}} \sim \mathcal{N}(2.2,0.5)$. We truncate SAT scores between 400 to 1600 and HS GPA between 0 to 4. For any applicant $t$, we randomly deploy assessment rule $\vtheta_t = [\theta_t^{\text{SAT}},\theta_t^{\text{HS GPA}}]^\top$ where $\theta_t^{\text{SAT}}\sim \mathcal{N}(1,10)$ and $\theta_t^{\text{HS GPA}}\sim \mathcal{N}(1,2)$. $\vtheta_t$ need not be zero-mean, so universities can play a reasonable assessment rule with slight perturbations while still being able to perform unbiased causal estimation. Components of the average effort conversion matrix $\E[\mathcal{E}_t]$ are smaller for disadvantaged applicants, which makes their mean improvement worse (see \Cref{fig:features-shift}). We set the expected effort conversion term $\E[\mathcal{E}_t] = \begin{pmatrix} 10 & 0 \\ 0 & 1 \end{pmatrix}$ for simplicity. Each row of $\E[\mathcal{E}_t]$ corresponds to effort expended to change a specific feature. For example, entries in the first row of $\E[\mathcal{E}_t]$ correspond to effort expended to change one's SAT score. For each applicant $t$, we perturb $\E[\mathcal{E}_t]$ with random noise drawn from $\mathcal{N}(0.5,0.25)$ to the top left entry and noise drawn from $\mathcal{N}(0.1,0.01)$ the bottom right entry to produce $\mathcal{E}_t$. We add this noise to $\E[\mathcal{E}_t]$ to produce $\mathcal{E}_t$ for advantaged applicants and subtract for disadvantaged applicants: thus, it takes more effort, on average, for members of disadvantaged groups to improve their HS GPA and SAT scores than members of advantaged groups. Finally, we construct college GPA (true outcome $y_t$) by multiplying observed features $\vx_t$ by the true effect parameters $\vtheta^*$. We then add confounding error $o_t$ where $o_t\sim\mathcal{N}(0.5,0.2)$ for disadvantaged applicants and $o_t\sim\mathcal{N}(1.5,0.2)$ for advantaged applicants. Disadvantage applicants could have lower baseline outcomes, e.g. due to institutional barriers or discrimination. While the setting we consider is simplistic, \Cref{fig:outcomes,fig:features-shift} demonstrate that our semi-synthetic admissions data behaves reasonably.\footnote{For example, the mean shift in SAT scores from the first to second exam is 46 points \cite{goodman2020sat}. In our data, the mean shift for disadvantaged and advantaged applicants is about 36 points and 91 points, respectively.}

\subsection{Experimental Details}\label{sec:answers}
We evaluate our model on a semi-synthetic dataset based on our running university admission example \cite{Dua:2019}. The dataset we base our experiments off of is publicly available at \url{www.openintro.org/data/index.php?data=satgpa}. This dataset does not contain personally identifiable information or offensive content. Since this is a publicly available dataset, no consent from the people whose data we are using was required. We ran our experiments on a 2020 MacBook Air laptop with 16GB of RAM. 
\section{Comparison with \citeauthor{shavit2020strategic}}\label{sec:shavit}

The setting most similar to ours is that of \citeauthor{shavit2020strategic}. They consider a strategic classification setting in which an agent's outcome is a linear function of features --some observable and some not (see \Cref{fig:shavit-model-dag} for a graphical representation of their model). While they assume that an agent's hidden attributes can be modified strategically, we choose to model the agent as having an unmodifiable private \emph{type}. Both of these assumptions are reasonable, and some domains may be better described by one model than the other. For example, the model of \citeauthor{shavit2020strategic} may be useful in a setting such as car insurance pricing, where some unobservable factors which lead to safe driving are modifiable. On the other hand, settings like our college admissions example in which the unobservable features which contribute to college success (i.e. socioeconomic status, lack of resources, etc., captured in $o_t$) are not easily modifiable.

One benefit of our setting is that we are able to use $\vtheta_t$ as a valid instrument to recover the true relationship $\vtheta^*$ between observable features and outcomes. This is generally not possible in the model of \cite{shavit2020strategic}, since $\vtheta_t$ violates the backdoor criterion as long as there exists any hidden features $\mathbf{h}_t$ and is therefore not a valid instrument. Another difference between our setting and theirs is that we allow for a heterogeneous population of agents, while they do not. Specifically, they assume that each agent's mapping from actions to features is the same, while our model is capable of handling mappings which vary from agent-to-agent.

A natural question is whether or not there exists a general model which captures the setting of both \citeauthor{shavit2020strategic} and ours. We provide such a model in \Cref{fig:general-model-dag}. In this setting, an agent has both observable and unobservable features, both of which are affected by the assessment rule $\vtheta_t$ deployed and the agent's private type $u_t$. However, much like the setting of \citeauthor{shavit2020strategic}, $\vtheta_t$ violates the backdoor criterion, so it cannot be used as a valid instrument in order to recover the true relationship between observable features and outcomes. Moreover, the following toy example illustrates that \emph{no} form of true parameter recovery can be performed when an agent's unobservable features are modifiable.

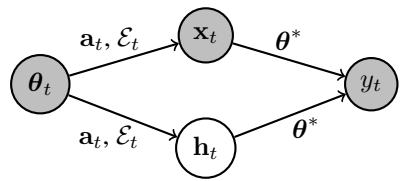
\begin{figure}
    \centering
    \begin{tikzpicture}[node distance={22mm}, thick, main/.style = {draw, circle}] 
         \node[draw,circle,fill=lightgray] (1) {$\vtheta_t$}; 
         \node[draw,circle,fill=lightgray] (2) [right of=1,yshift=0.65cm] {$\vx_t$}; 
         \node[node distance={15mm}, thick,draw,circle] (3) [below of=2] {$\mathbf{h}_t$};
         \node[draw,circle,fill=lightgray,,yshift=-0.65cm] (4) [right of=2] {$y_t$};
         \draw[->] (1) -- node[above left,xshift=0.35cm] {$\va_t$, $\mathcal{E}_t$} (2);
         \draw[->] (1) -- node[below left,xshift=0.35cm] {$\va_t$, $\mathcal{E}_t$} (3);
         \draw[->] (3) -- node[below right=-0.15cm] {$\vtheta^*$} (4);
         \draw[->] (2) -- node[above] {$\vtheta^*$} (4);
    \end{tikzpicture} 
    \caption{Graphical model of \citeauthor{shavit2020strategic}. Observable features $\vx_t$ (e.g. the type of car a person drives) and unobservable features $\mathbf{h}_t$ (e.g. how defensive of a diver someone is) are affected by $\vtheta_t$ through action $\va_t$ (e.g. buying a new car) and common action conversion matrix $\mathcal{E}$ (representing, in part, the cost to a person of buying a new car). Outcome $y_t$ (in this example, the person's chance of getting in an accident) is affected by $\vx_t$ and $\mathbf{h}_t$ through the true causal relationship $\vtheta_t$. Note that causal parameter recovery is not possible in this setting unless all features are observable.}
    \label{fig:shavit-model-dag}
\end{figure}

\begin{figure}
    \centering
    \begin{tikzpicture}[node distance={25mm}, thick, main/.style = {draw, circle}] 
         \node[draw,circle,fill=lightgray] (1) {$\vtheta_t$};
         \node[draw,circle,fill=lightgray] (2) [right of=1] {$\vx_t$}; 
         \node[node distance={20mm},draw,circle,fill=lightgray,yshift=-1.2cm] (3) [right of=2] {$y_t$};
         \node[main] (4) [below of=1] {$u_t$};
         \node[main] (5) [right of=4] {$\mathbf{h}_t$};
         \draw[->] (1) -- node[above] {$\va_t$} (2);
         \draw[->] (2) -- node[above] {$\vtheta^*$} (3);
         \draw[->] (1) -- node[above right,yshift=0.5cm,xshift=-0.55cm] {$\va_t$} (5);
         \draw[->] (4) -- node[above,yshift=-.7cm,xshift=-1cm] {$\vb_t,\mathcal{E}_t$} (2);
         \draw[->] (5) -- node[below] {$\vtheta^*$} (3);
         \draw[->] (4) edge[bend right=2cm] node [below right]{$o_t$} (3);
         \draw[->] (4) -- node[above] {$\vb_t,\mathcal{E}_t$} (5);
    \end{tikzpicture} 
    \caption{Graphical model which captures both our setting and that of \citeauthor{shavit2020strategic}. In this setting, observable features $\vx_t$ \emph{and} unobservable features $\mathbf{h}_t$ are affected by $\vtheta_t$ through action $\va_t$. The agent's private type $u_t$ affects $\vx_t$ and $\mathbf{h}_t$ through initial feature values $\vb_t$ and action conversion matrix $\mathcal{E}_t$. The agent's outcome $y_t$ depends on $\vx_t$ and $\mathbf{h}_t$ through the causal relationship $\vtheta^*$ and $u_t$ through confounding term $o_t$. Note that much like the setting of \cite{shavit2020strategic}, causal parameter recovery is not possible in this setting unless all features are observable.}
    \label{fig:general-model-dag}
\end{figure}
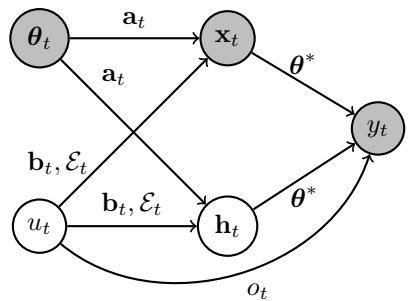

\begin{example}
    Consider the one-dimensional setting 
    \begin{equation*}
        y_t = \theta^*x_t + \beta^* h_t, 
    \end{equation*}
    
    where $x_t$ is an agent's observable, modifiable feature and $h_t$ is an unobservable, modifiable feature. If the relationship between $x_t$ and $h_t$ is unknown, then it is generally impossible to recover the true relationship between $x_t$, $h_t$, and outcome $y_t$. To see this, consider the setting where $h_t$ and $x_t$ are highly correlated. In the extreme case, take $h_t = x_t, \; \forall t$. (Note we use equality to indicate identical feature values, not a causal relationship.) In this setting, the models $\theta^* = 1$, $\beta^* = 1$ and $\theta^* = 2$, $\beta^* = 0$ produce the same outcome $y_t$ for all $x \in \mathbb{R}$, making it impossible to distinguish between the two models, even in the limit of infinite data.
\end{example}
\section{Setting of SGD comparison}\label{sec:figure-setting}
In 1D, the derivative of $\mathbb{E}[(\widehat{y}_t - y_t)]^2$ which accounts for $x_t$ and $y_t$'s dependence on $\vtheta_t$ takes the form 
\begin{equation*}
\begin{aligned}
    \Delta &= 2(\mathbb{E}[(\widehat{y}_t - y_t) x_t] + \mathbb{E}_{x_t}[\widehat{y}_t - y_t]\mathcal{E}^2 (\theta_t - \theta^*)).
\end{aligned}
\end{equation*}
By plugging in for $x_t$, $y_t$, $\widehat{y}_t$ and simplifying, we can write the derivative as 
\begin{equation*}
\begin{aligned}
    \Delta &= 2 \left(\mathbb{E}[b_t^2] + \mathcal{E}^4 \theta_t^2 (\theta_t - \theta^*) - \mathbb{E}[o_t b_t] + \mathcal{E}^4 \theta_t (\theta_t - \theta^*)^2 \right).
\end{aligned}
\end{equation*}

The derivative of $\mathbb{E}[(\widehat{y}_t - y_t)]^2$ which does \emph{not} account for $x_t$ and $y_t$'s dependence on $\vtheta_t$ can be written as 

\begin{equation*}
\begin{aligned}
    \Delta' &= 2 \left(\mathbb{E}[(\widehat{y}_t - y_t) x_t] \right)\\
    &= 2 \left(\mathbb{E}[b_t^2] + \mathcal{E}^4 \theta_t^2 (\theta_t - \theta^*) - \mathbb{E}[o_t b_t] \right).
\end{aligned}
\end{equation*}

As can be seen by comparing the two equations, there is an extra $\mathcal{E}^4 \theta_t (\theta_t - \theta^*)^2$ term present in $\Delta$ that is not in $\Delta'$. This can cause $\Delta$ and $\Delta'$ to have opposite signs under certain scenarios, e.g. when $\mathbb{E}[b_t^2] + \mathcal{E}^4 \theta_t^2 (\theta_t - \theta^*) - \mathbb{E}[o_t b_t]$ is negative and $\mathcal{E}^4 \theta_t (\theta_t - \theta^*)^2$ is sufficiently large. To generate Figure \ref{fig:pp-vs-us}, we set $\mathbb{E}[b_t] = 0$, $\mathbb{E}[b_t^2] = 0.3$, $\mathbb{E}[o_t] = 0$, $\mathbb{E}[g^2_t] = 15$, $\mathbb{E}[o_t b_t] = -6.5$, $\mathcal{E} = 3$, $\theta^* = 1$, and $\theta_0 = 0.5$. To generate Figure \ref{fig:SGD-conv-invex} and \ref{fig:SSGD-conv-invex}, we changed $\theta^*$ to be $0.7$.

\begin{figure}
    \centering
    \includegraphics[width=0.5\linewidth]{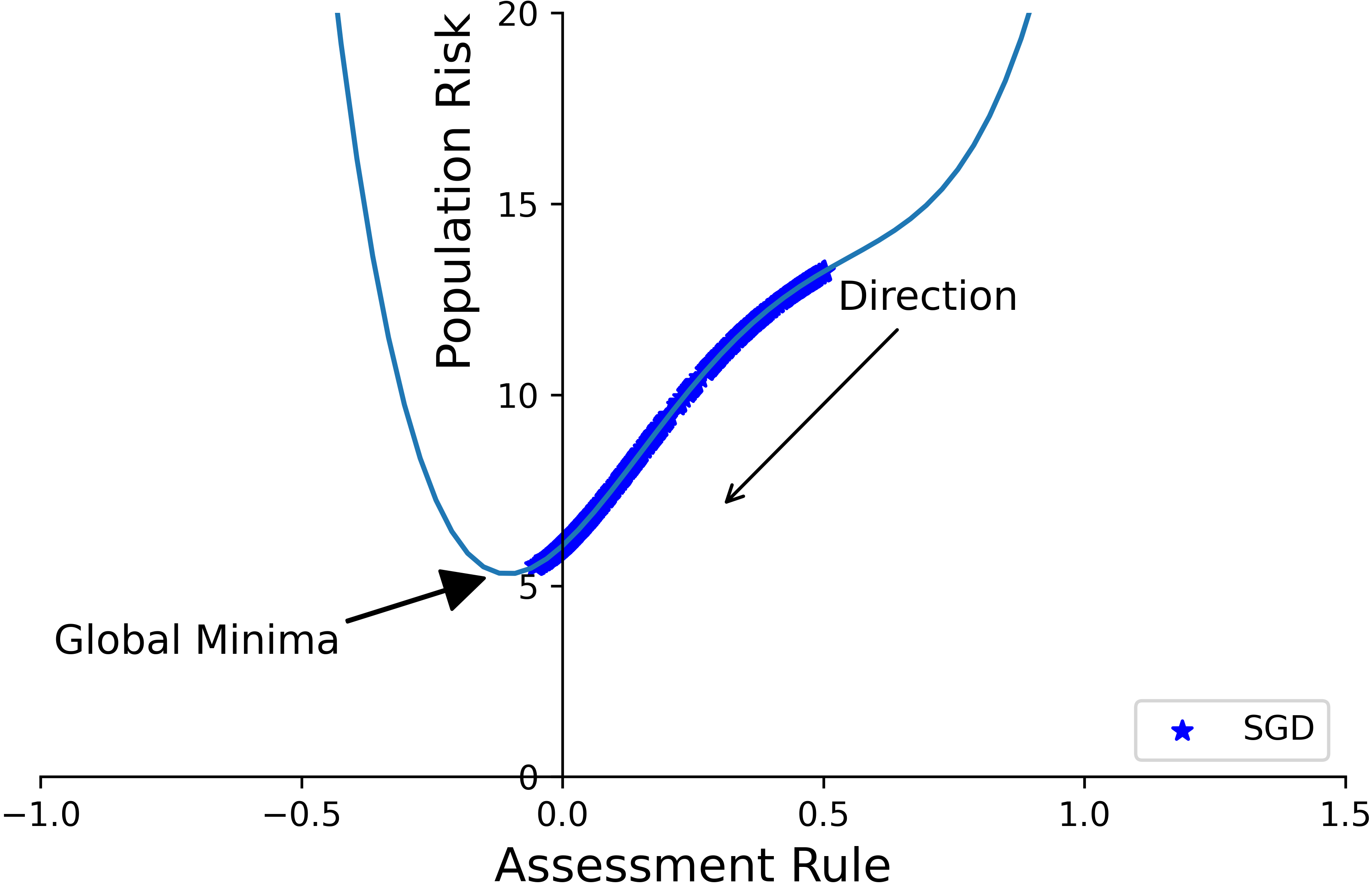}
    \caption{SGD with invex function}
    \label{fig:SGD-conv-invex}
\end{figure}

\begin{figure}
    \centering
    \includegraphics[width=0.5\linewidth]{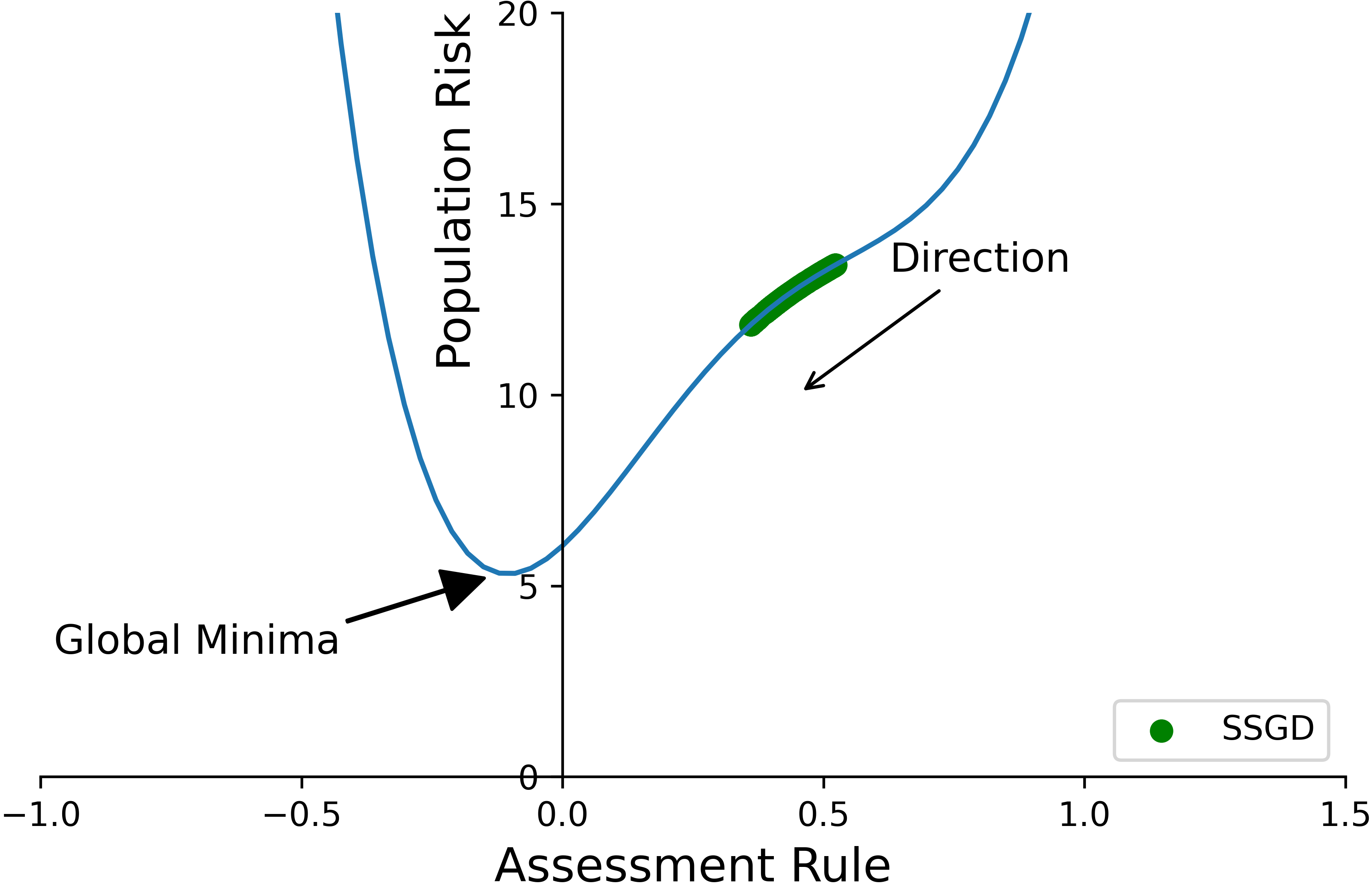}
    \caption{SSGD with invex function}
    \label{fig:SSGD-conv-invex}
\end{figure}

\begin{figure}
    \centering
    \includegraphics[width=0.5\textwidth]{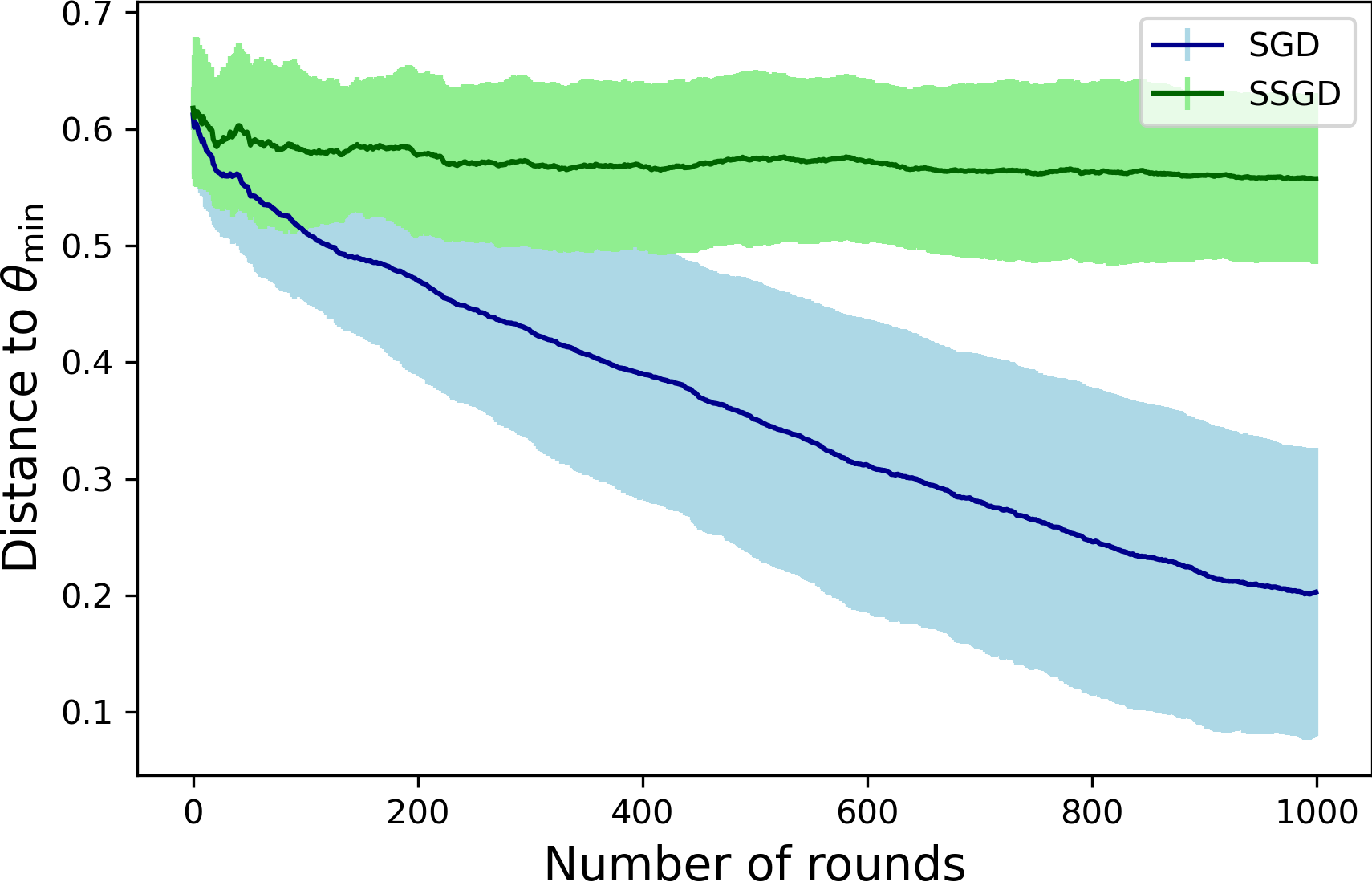}
    \caption{Convergence rate of Stochastic Gradient Descent vs Simple Stochastic Gradient Descent for simple 1D setting. Even when SSGD converges, it may do so at a much slower rate, due to the inexact measure of the gradient. We ran both methods for $1000$ time-steps with a decaying learning rate of $\frac{0.001}{\sqrt{T}}$. Results are averaged over 10 runs, with the error bars (in lighter colors) representing one standard deviation.}
    \label{fig:pp-vs-us-convergence-rate}
\end{figure}

\end{document}